\documentclass{article}

\usepackage[nonatbib,preprint]{neurips_2025}

\usepackage[utf8]{inputenc} % allow utf-8 input
\usepackage[T1]{fontenc}    % use 8-bit T1 fonts
\usepackage{hyperref}       % hyperlinks
\usepackage{url}            % simple URL typesetting
\usepackage{booktabs}       % professional-quality tables
\usepackage{amsfonts}       % blackboard math symbols
\usepackage{nicefrac}       % compact symbols for 1/2, etc.
\usepackage{microtype}      % microtypography
\usepackage{xcolor}         % colors
\usepackage{pifont}
\newcommand{\cmark}{\ding{51}} % ✓
\newcommand{\xmark}{\ding{55}} % ✗
\usepackage{amsthm}
\usepackage[utf8]{inputenc} % allow utf-8 input
\usepackage[T1]{fontenc}    % use 8-bit T1 fonts
\usepackage{hyperref}       % hyperlinks
\usepackage{url}            % simple URL typesetting
\usepackage{booktabs}       % professional-quality tables
\usepackage{amsfonts}       % blackboard math symbols
\usepackage{nicefrac}       % compact symbols for 1/2, etc.
\usepackage{microtype}      % microtypography
\usepackage{xcolor}         % colors
\usepackage[ruled,vlined]{algorithm2e}
\usepackage{hyperref}
\usepackage{url}
\usepackage{graphicx}
\usepackage{algorithmic}
\usepackage{commath}
\usepackage{amsmath}
\usepackage{tabularray}
\usepackage{xcolor}
\usepackage{soul}
\usepackage{subcaption}
\usepackage{graphicx}
\usepackage{rotating}
\usepackage{adjustbox}
\usepackage{wrapfig}
\usepackage{enumitem}
\usepackage[utf8]{inputenc} % allow utf-8 input
\usepackage[T1]{fontenc}    % use 8-bit T1 fonts         % 
\usepackage{booktabs}       % professional-quality tables
\usepackage{amsfonts}       % blackboard math symbols
\usepackage{nicefrac}       % compact symbols for 1/2, etc.
\usepackage{microtype}      % microtypography
\usepackage{xcolor}         % colors
\usepackage{graphicx}
\usepackage{float}
\usepackage[backend=biber,style=numeric,sorting=none]{biblatex}
\usepackage{multirow}
\theoremstyle{remark}
\newtheorem{remark}{Remark}
\usepackage{longtable}
\usepackage{tabulary}
\usepackage{booktabs,array,multirow}

\newtheorem{lemma}{Lemma}
\newtheorem{theorem}{Theorem}
\newtheorem{corollary}{Corollary}

\newcommand{\V}{{\bf V}}

\title{GRAFT: Gradient-Aware Fast MaxVol Technique for Dynamic Data Sampling}

\author{%
  *Ashish Jha$^1$ \quad
  *Anh huy Phan$^1$ \quad
  Razan Dibo$^1$ \quad
  Valentin Leplat$^2$ \\
  $^1$Skolkovo Institute of Science and Technology \quad
  $^2$Innopolis University \quad\\
  \texttt{\{ashish.jha,a.phan,razan.dibo\}@skoltech.ru, v.leplat}@innopolis.ru
}
\begin{document}

\maketitle

\begin{abstract}
Training modern neural networks on large datasets is computationally and environmentally costly. We introduce GRAFT, a scalable in-training subset selection method that (i) extracts a low-rank feature representation for each per-iteration batch, (ii) applies a Fast MaxVol sampler to pick a small, diverse subset that spans the batch’s dominant subspace, and (iii) dynamically adjusts the subset size using a gradient-approximation criterion. 
By operating in low-rank subspaces and training on carefully chosen examples instead of full batches, GRAFT preserves the training trajectory while reducing wall-clock time, energy, and $\mathrm{CO}_2$ emissions. 
Across multiple benchmarks, it matches or exceeds recent selection baselines in accuracy and efficiency, while providing a favorable accuracy efficiency emissions trade-off.
\end{abstract}

\section{Introduction}
Deep neural networks are powerful but costly to train, requiring substantial computational resources and energy, raising scalability and environmental concerns. A key driver of this cost is the large volume of training data, which increases optimization steps, memory use, and training time. Data subset selection offers a solution by using a smaller, well-chosen subset to maintain training efficiency while reducing redundancy. This can speed up training, cut energy use, and aid resource-constrained settings. However, many existing methods have drawbacks, such as costly preprocessing, reliance on proxy models, architecture restrictions, or repeated metric evaluations, limiting their flexibility.

In this work, we introduce GRAFT (Gradient-Aware Fast MaxVol Technique), a scalable and lightweight framework designed to integrate subset selection directly into the training loop. GRAFT operates in two main stages - first, it extracts compact feature embeddings from per-iteration batch using a low-rank projection. This step ensures that the high-dimensional representation of the data is mapped onto a lower-dimensional subspace, capturing the most salient features with minimal redundancy. Next, GRAFT performs a fast MaxVol sampling to select a subset of examples that best span the dominant subspace of the projected features. Unlike heuristic-based selection strategies, the MaxVol criterion guarantees that the chosen subset retains maximal expressiveness in terms of subspace coverage. To maintain alignment with the training dynamics, GRAFT introduces an adaptive sampling mechanism that dynamically adjusts the number of selected samples per-iteration batch. This adjustment is guided by measuring the alignment between the gradient computed over the per-iteration subset and the gradient of the entire per-iteration batch. Specifically, GRAFT evaluates the subspace representation of the selected subset to ensure it remains representative of the full batch's learning dynamics. If the alignment deviates, GRAFT automatically increases the subset size to better capture the gradient direction. Conversely, when alignment is strong, GRAFT reduces the subset size to optimize efficiency. This allows GRAFT to preserve training trajectory fidelity, ensuring that critical gradient information is not lost during optimization, even as subset compositions evolve throughout training. Furthermore, GRAFT introduces a fundamentally different mechanism compared to existing subset selection strategies such as GradMatch. GradMatch focuses on gradient matching, where the primary goal is to select a subset $\mathcal{S}$ such that the gradient computed over $\mathcal{S}$ closely approximates the gradient of the entire dataset. This is achieved through a greedy orthogonal matching pursuit (OMP) strategy that incrementally selects samples to minimize the residual gradient error. While effective, gradient matching methods require explicit comparisons of gradient vectors, which can be computationally expensive and sensitive to noise.

In contrast, GRAFT shifts the focus from direct gradient matching to per-iteration subspace approximation combined with gradient alignment. Importantly, GRAFT ensures that this subspace representation remains gradient-aligned with the full per-iteration batch by dynamically adjusting the subset size based on the measured alignment. This mechanism guarantees that the chosen samples not only capture the most informative directions in the data but also maintain the critical gradient signals necessary for effective optimization. As a result, GRAFT achieves improved efficiency by reducing reliance on full-gradient computations while preserving the quality of the training trajectory, leading to more robust and stable convergence.

\begin{figure}[htbp]
\begin{center}
\includegraphics[width=1.0\textwidth]{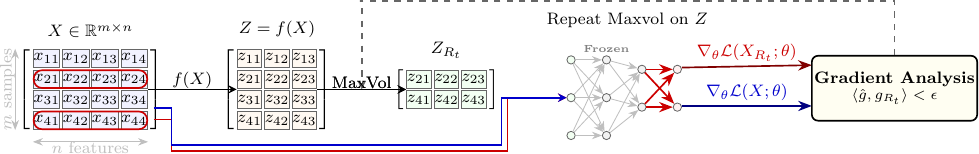}
\end{center}
\caption{Overview of the \textsc{GRAFT} sampling framework. The method consists of two main stages: (1) \textbf{Feature Extraction and Sample Selection}, where input per-iteration batch $X$ is transformed into a low-dimensional feature matrix $Z = f(X)$, followed by MaxVol-based selection of representative rows (samples); and (2) \textbf{Gradient Alignment and Rank Selection}, where the selected subset is evaluated based on how well its gradient aligns with the full per-iteration batch gradient. The selection rank $R$ is dynamically adjusted using a projection error criterion. The final output is a subset of samples indexed by $I = \{i_1, \ldots, i_R\}$ used for gradient updates.}
\label{fig:training}
\end{figure}

This design enables GRAFT to operate efficiently within small localized data, using the local structure to guide training. By selecting representative samples in each iteration, GRAFT reduces the computational effort in subset selection while adapting to the evolving optimization landscape throughout the training. Our experimental results demonstrate that GRAFT achieves strong performance across multiple datasets and architectures, reducing both training time and energy usage with minimal accuracy degradation. These findings position GRAFT as a practical tool for sustainable and efficient deep learning on a scale.

\section{Related Works}\label{Relat_Sec}

Subset selection methods can be broadly categorized into pre-selection and integrated selection approaches. Pre-selection methods identify representative subsets before training. A notable example is EL2N (Early-Loss-to-Norm) scoring, defined as $\text{EL2N}(x_i) = \frac{1}{t}\sum_{k=1}^t |\nabla L(\theta_k; x_i)|^2$, which leverages early-phase gradient norms to approximate sample informativeness \cite{paul2021el2n}. Samples with higher early gradient norms tend to have a stronger impact on parameter updates and loss convergence. Forget Score \cite{toneva2018forget} adopts a different strategy, measuring the misclassification frequency of samples as $F(x_i) = \sum_{k=1}^t \mathbb{1}[y_i \neq \hat{y}_i^{(k)}]$, which favors harder-to-learn examples. However, this method is highly sensitive to noise and class imbalance, potentially overemphasizing mislabeled or ambiguous samples. Selection via Proxy (SVP) \cite{coleman2020selection} introduces a lightweight proxy model $f_\phi$ trained independently to estimate sample informativeness, formulated as $\text{SVP}(x_i) = L(f_\phi(x_i), y_i)$. Although efficient, its effectiveness is tightly coupled to the fidelity of the proxy model, which may not always generalize well to the main model’s learning dynamics. DRoP \cite{vysogorets2025drop} conceptualizes subset selection as a distributionally robust optimization (DRO) problem, expressed as $\min_{\theta} \max_{q \in \mathcal{Q}} \sum_{i=1}^{n} q_i L(f_\theta(x_i), y_i)$. This adversarial perspective enhances robustness against noise, but it introduces significant computational overhead due to adversarial evaluations during each optimization step. One inherent limitation of pre-selection methods is their inability to adapt to evolving training dynamics, often resulting in underfitting or misalignment with the model's learning trajectory. In contrast, integrated selection methods address this shortcoming by adjusting subsets throughout training, aligning selection with gradient-aligned subspace approximations that reflect the current optimization landscape. This adaptability is especially critical in non-convex optimization settings where data importance fluctuates across iterations. Prominent examples include GradMatch \cite{killamsetty2021gradmatch}, which minimizes the gradient discrepancy $|\nabla_\theta \mathcal{L}(S) - \nabla_\theta \mathcal{L}(\mathcal{X})|_2^2$. Despite its theoretical alignment with full-batch gradients, GradMatch suffers from substantial computational overhead during optimization. GLISTER \cite{killamsetty2021glister} and CRAIG \cite{mirzasoleiman2020coresets} frame the subset selection problem as a submodular optimization task, which allows efficient greedy selection but lacks adaptive flexibility to dynamic gradient changes. GRAFT builds upon these limitations by improving scalability and optimization fidelity. It samples from a low-rank subspace, maximizing the spanned volume through Fast MaxVol-based selection, and dynamically adjusts subset sizes to preserve representational capacity throughout training. Empirical evaluations using the CORDS framework demonstrate that methods like GradMatch and GLISTER consistently encounter out-of-memory errors with large batch sizes, highlighting their scalability constraints. Other relevant techniques include BADGE \cite{ash2020badge}, which combines uncertainty with gradient embedding clustering to target informative samples. Moderately hard core-sets \cite{xia2023moderate} aim to select samples of intermediate difficulty, striking a balance between learnability and informativeness. Dataset distillation \cite{zhao2021dataset} attempts to synthesize compact datasets but struggles to maintain representativeness at moderate subset sizes. Hybrid strategies like SelMatch \cite{lee2024selmatch} refine real data selections through gradient matching, conceptually similar to GRAFT, although they require additional meta-optimization layers. SubSelNet \cite{jain2023efficient}, designed specifically for NAS, learns generalized selection policies. However, its reliance on prior task-architecture distributions constrains its generalizability to out-of-distribution scenarios.

% \begin{table}[H]
% \centering
% \small
% \resizebox{0.95\linewidth}{!}{%
% \begin{tabular}{lccccc}
% \toprule
% \textbf{Method} & \textbf{Batch-wise} & \textbf{Gradient-based} & \textbf{Complexity} & \textbf{Dynamic Rank} & \textbf{Memory Usage} \\
% \midrule
% DRoP~\cite{vysogorets2025drop}            & \xmark & \xmark & \( \mathcal{O}(n^2 d) \)     & \xmark & High \\
% GradMatch~\cite{killamsetty2021gradmatch} & \cmark & \cmark & \( \mathcal{O}(n k d) \)       & \xmark & High \\
% CRAIG~\cite{mirzasoleiman2020coresets}       & \cmark & \cmark & \( \mathcal{O}(n k d) \)     & \xmark & Moderate \\
% GLISTER~\cite{killamsetty2021glister}     & \cmark & \xmark & \( \mathcal{O}(n k d) \)     & \xmark & Moderate \\
% SVP~\cite{coleman2020selection}           & \xmark & \xmark & \( \mathcal{O}(n d') \)      & \xmark & Low \\
% SelMatch~\cite{lee2024selmatch}           & \xmark & \cmark & \( \mathcal{O}(T R d + R^2) \) & \xmark & Moderate \\
% SubSelNet~\cite{jain2023efficient}        & \xmark & \xmark & \( \mathcal{O}(n^2 d) + \mathcal{O}(\text{meta-learn}) \) & \xmark & High \\
% BADGE~\cite{ash2020badge}                 & \xmark & \cmark & \( \mathcal{O}(n d) \)       & \xmark & High \\
% \textbf{GRAFT (Ours)}                     & \cmark & \cmark & \( \mathcal{O}(K d R + R^3) \) & \cmark & Low \\
% \bottomrule
% \end{tabular}
% }
% \caption{Comparison of subset selection methods. \( n \): dataset size, \( d \): feature dimension, \( K \): batch size, \( R \): selected rank/size, \( d' \): proxy model dim, \( T \): distillation steps.}
% \label{tab:comparison}
% \end{table}
\begin{table}[h]
\centering
\scriptsize
\resizebox{0.95\linewidth}{!}{%
\begin{tabular}{lcccccc}
\toprule
\textbf{Method} & \textbf{Batch-wise} & \textbf{Gradient-based} & \textbf{Complexity} & \textbf{Scalability Bottleneck} & \textbf{Dynamic Rank} & \textbf{Memory Usage} \\
\midrule
DRoP & \xmark & \xmark & $\mathcal{O}(n^2d)$ & Quadratic in $n$ (intractable for $n > 10^4$) & \xmark & High \\
GradMatch & \cmark & \cmark & $\mathcal{O}(nkd)$ & Linear in $n$ (full gradient comparisons) & \xmark & High \\
CRAIG & \cmark & \cmark & $\mathcal{O}(nkd)$ & Linear in $n$ (submodular optimization) & \xmark & Moderate \\
GLISTER & \cmark & \xmark & $\mathcal{O}(nkd)$ & Linear in $n$ (bi-level optimization) & \xmark & Moderate \\
SVP & \xmark & \xmark & $\mathcal{O}(nd')$ & Proxy fidelity (task-dependent) & \xmark & Low \\
BADGE & \xmark & \cmark & $\mathcal{O}(nd)$ & High-$d$ gradient clustering & \xmark & High \\
SubSelNet & \xmark & \xmark & $\mathcal{O}(n^2d)$ + meta-learning & Quadratic in $n$ + training overhead & \xmark & High \\
\textbf{GRAFT (Ours)} & \cmark & \cmark & $\mathcal{O}(KR^2 + |R_{\text{set}}| \cdot Rd)$ & Linear in $K$,$|R_{\text{set}}|$ \& Quadratic in $R$ & \cmark & \textbf{Low} \\
\bottomrule
\end{tabular}%
}
\caption{Comparison of subset selection methods. $n$: dataset size, $d$: feature dimension, $K$: batch size, $R$: selected rank/size, $d'$: proxy model dimension.}
\label{tab:comparison}
\end{table}

\paragraph{Preliminaries}
\label{sub:preliminaries}
Let $\mathcal{X} = \{x_i\}_{i=1}^n \subset \mathbb{R}^d$ be a training dataset partitioned into $B$ mini-batches $\{\mathcal{X}_i\}_{i=1}^B$ of size $K \ll n$. For a batch $\mathcal{X}_i$, the gradient matrix $G \in \mathbb{R}^{d \times K}$ is defined as $G = [\nabla_\Theta L(\Theta; x_1), \dots, \nabla_\Theta L(\Theta; x_K)]$. We project $\mathcal{X}_i$ into a low-rank embedding $V = f(\mathcal{X}_i) \in \mathbb{R}^{K \times R}$ ($R \ll d$) and use Fast Maxvol to select $R$ rows $\mathbf{p} = [p_1, \dots, p_R]$ maximizing submatrix volume. The subset $G_R = G(:, \mathbf{p})$ approximates the full-batch gradient $\bar{g}$ via $\tilde{g} = \frac{1}{R}\sum_{r=1}^R G(:,p_r)$, with $\|\bar{g} - \tilde{g}\|_2 \leq \varepsilon$.

\section{Data Point selection via Row Sampling}\label{Sec:Maxvol}
The proposed methodology capitalizes on data point sampling by organizing the training dataset into a matrix \( \mathbf{X} \in \mathbb{R}^{n \times d} \), where each row \( \mathbf{x}_i \) (for \( i = 1, \dots, n \)) represents a distinct data point in the \( d \)-dimensional feature space. Relevant data points are identified by selecting a subset of rows from this matrix, a process facilitated by the Nyström method or analogous row/column subset selection techniques \cite{mathur2023column}. In contrast to traditional low-rank matrix approximation approaches, which focus on global matrix factorization, the column/row sampling strategy specifically seeks to identify a submatrix that preserves the critical structural information and intrinsic relationships embedded within the original matrix. A prevalent technique in this domain is the Cross-2D method \cite{tyrtyshnikov2000incomplete}, which leverages the Maxvol algorithm \cite{goreinov2010good} to iteratively select both rows and columns in a manner that maximizes the volume of the submatrix formed at each step. Let \( \mathcal{S}_r \subseteq \{1, \dots, n\} \) and \( \mathcal{S}_c \subseteq \{1, \dots, d\} \) denote the set of row and column indices selected at any given iteration, respectively. The procedure initiates by selecting columns \( \mathbf{x}_j \) that maximize the volume of the submatrix \( \mathbf{X}[\mathcal{S}_r, \mathcal{S}_c] \) formed by previously selected rows. Subsequently, rows are chosen iteratively based on the submatrix formed by the selected columns, thus refining the quality of the approximation in each iteration. This iterative process significantly enhances the matrix approximation by progressively selecting rows and columns that optimally capture the underlying structure of the data, minimizing the error in the low-rank approximation and ensuring computational efficiency. The result is a reduced-dimensional representation that maintains the essential properties of the original dataset while mitigating the computational burden associated with handling large-scale data matrices. Although the Cross-2D method is a well-established technique, the concurrent selection of both rows and columns introduces significant complexity into the selection process. Moreover, its performance is highly sensitive to the initial choice of row or column indices, with different initializations potentially yielding divergent results. Our experimental results, as detailed in Section~\ref{Sec:exp} and illustrated in Figure~\ref{fig:ablation}, demonstrate that this method did not perform favorably in the context of our study.
\subsection{Sampling of Ordered Extracted Features} 

% Our approach simplifies the process to row sampling. Instead of sampling from the original matrix, we perform fast sampling on a feature matrix, where columns represent features ranked in descending order of their importance. This involves two  steps: feature extraction and sample selection.

Our approach simplifies the process by focusing solely on row sampling. Rather than sampling directly from the original matrix, we conduct accelerated sampling on a feature matrix, where the columns correspond to features ordered in descending importance. This methodology unfolds in two distinct stages, feature extraction and sample selection.

\textbf{Step 1: Feature Extraction} Let \( A \in \mathbb{R}^{K \times M} \) denote a batch consisting of \( K \) training samples, where each sample lies in an \( M \)-dimensional input space. A feature extraction function \( f: \mathbb{R}^{K \times M} \to \mathbb{R}^{K \times R} \) is applied to obtain a compact feature representation, producing the matrix \( V = f(A) \), where \( V \in \mathbb{R}^{K \times R} \) and \( R \ll M \) is the number of retained features. Each row \( \mathbf{v}_k = V_{k,:} \in \mathbb{R}^R \) represents the feature vector associated with the \( k \)-th sample. The columns of \( V \) are ordered such that the features appear in descending order of relevance, with the most informative dimensions occupying the leftmost columns. Let \( \text{Rel}(j) \) denote a feature relevance score for column \( j \), computed using metrics such as variance, mutual information with the label space, or correlation. The ordering satisfies
$
\text{Rel}(1) \geq \text{Rel}(2) \geq \dots \geq \text{Rel}(R).
$
The feature extraction function \( f \) can be instantiated via several methods. One approach involves applying singular value decomposition (SVD) to \( A \), yielding \( A = U \Sigma V^\top \), and selecting the top \( R \) left singular vectors \( U_R \in \mathbb{R}^{K \times R} \) to form the feature matrix. Alternatively, \( f \) may correspond to a nonlinear mapping \( \phi: \mathbb{R}^M \to \mathbb{R}^R \), such as the encoder of a trained autoencoder or a shallow neural network. Regardless of the specific implementation, the goal is to ensure that \( V \) captures the dominant structural components of the data in a low-dimensional space suitable for efficient sample selection.

\textbf{Step 2: Sample Selection using Fast MaxVol Method}
Given the feature matrix \( \mathbf{V} \in \mathbb{R}^{K \times R} \), the goal is to sequentially select sample indices \( p_r \) from the top-\( R \) features such that the submatrix formed by the selected samples maximizes the volume, capturing the most significant information. 
One can apply the ``conventional'' Maxvol algorithm \cite{goreinov2010good},
which iterates the row index selection until the entries of the interpolation matrix are below a predefined threshold (often close to 1). Alternatively, we show a fast Max-volume algorithm with $R$ iterations to select $R$ row indices.
Let \( \mathbf{p} = [p_1, \dots, p_R] \) denote the set of selected row indices.

\begin{itemize}[leftmargin=2mm]
    \item {\bf Select the first data point index:}
Start with the first column \( \mathbf{v}_1 = \mathbf{V}(:, 1) \) and select the index \( p_1 \) corresponding to the element with the maximum absolute value, i.e., the most dominant entry in the most relevant feature
$
    p_1 = \arg\max_{i}\,|{\bf v}_1(i)|\, \label{eq_maxvol_p1}
$

\item \textbf{Selecting subsequent data point indices:}  
The second data point index, \( p_2 \), is selected such that the submatrix \( \V([p_1, p_2], [1, 2]) \) has the maximum volume, i.e., maximum  absolute value of the determinant 
$p_2 = \arg\max_i |\text{det}(\V([p_1, i], [1, 2]))| 
\label{eq_maxvol_p2}
$ 
where the determinant of the submatrix $\V([p_1, i], [1, 2])$ is computed as 
\begin{align}
\text{det}(\V([p_1, i], [1, 2])) &= {\bf v}_{p_1,1} {\bf v}_{i,2} - {\bf v}_{p_1,2} {\bf v}_{i,1} \notag \\
% &= {\bf v}_{p_1,1} ({\bf v}_{i,2} - {\bf v}_{i,1} {\bf v}_{p_1,1}^{-1} {\bf v}_{p_1,2})  \notag \\
&= {\bf v}_{p_1,1} {\bf r}_{i,1}  \label{eq_det_p2b}
\end{align}
Here the residual \(
{\bf r}_1={\bf v}_2- {\bf v}_1\left({\bf v}_{p_1,1}\right)^{-1} {\bf v}_{p_1,2}\). The above projection nullifies the $p_1$ element in  ${\bf r}_1$, i.e., ${\bf r}_1(p_1) = 0$.
From (\ref{eq_det_p2b}), finding the index $i$, which maximizes the volume of $\V([p_1, i], [1, 2]))$ in (\ref{eq_maxvol_p2}) is equivalent to identifying the maximum absolute element of ${\bf r}_1$, i.e., $p_2 = \arg\max_{i}\,|{\bf r}_1(i)|$. We next update the selected index set ${\bf p}=[p_1,p_2]$, and proceed with finding the other indices.

Assume the selected index set ${\bf p} =[p_1,p_2,\ldots,p_{j-1}]$, we need to select the $j$-th column index such that the submatrix ${\bf V}([{\bf p}, p_j],[1:j])$ of $j$ rows and the first $j$ columns has the maximum volume, i.e., 
    \[
    p_j = \arg\max_{i} \, \left| \det \big( \mathbf{V}([\mathbf{p}, i], [1:j]) \big) \right|.
    \]
Suppose that  $\V({\bf p}, [1:j-1])$ is invertible, we define the residual term \(
{\bf r}_j={\bf v}_j- {\bf V}(:,[1:j-1])\left({\bf V}({\bf p},[1:j-1])\right)^{-1}  {\bf v}_{{\bf p},j} \).    By applying Sylvester's determinant theorem to the block matrix
    \[
    \mathbf{V}([\mathbf{p}, i], [1:j]) =
    \begin{bmatrix}
    \mathbf{V}(\mathbf{p}, [1:j-1]) & \mathbf{v}_{\mathbf{p}, j} \\
    \mathbf{v}_{i, 1:j-1} & \mathbf{v}_{i, j}
    \end{bmatrix},
    \]
    the determinant simplifies to
    $
    \det \big( \mathbf{V}([\mathbf{p}, i], [1:j]) \big) =
    \det \big( \mathbf{V}(\mathbf{p}, [1:j-1]) \big) \cdot \mathbf{r}_j(i).
    $
    Hence, finding the most relevant index \( p_j \) reduces to selecting the index corresponding to the maximum absolute value of ${\bf r}_j$
    $p_j = \arg\max_{i} \, |\mathbf{r}_j(i)|.$
\end{itemize}

\begin{remark}[Gradient Approximation Guarantee]
\label{remark:gradient-approx}
Let $A \in \mathbb{R}^{K \times M}$ be a batch matrix, and let $V = U_R \in \mathbb{R}^{K \times R}$ denote the matrix of the top-$R$ left singular vectors of $A$. Let $S \subset [K]$, $|S|=R$, be the subset of rows selected by the MaxVol algorithm applied to $V$. Since $V$ captures the dominant subspace of $A$, the MaxVol-selected rows in $V$ correspond to influential samples in $A$, ensuring their gradients approximate the full-batch gradient.If the gradient map $g(x) = \nabla_{\theta} L(\theta; x)$ is $L_g$-Lipschitz continuous, then:
$
\left\|\nabla_{\theta} L(\theta; A) - \nabla_{\theta} L(\theta; A(S, \cdot))\right\|_2 \leq \frac{K}{R} L_g \sigma_{R+1},
$
where $\sigma_{R+1}$ is the $(R+1)$-th singular value of $A$.
\end{remark}

\subsection{Dynamic Gradient based refinement}
% \begin{wrapfigure}{r}{0.5\textwidth}
% \begin{algorithm}[H]
% \caption{Gradient-Aligned Subset Training (Full Algorithm in Appendix)}
% \label{ALG:Main}
% \DontPrintSemicolon
% \KwIn{Training data $\mathcal{X}$, ranks $\mathbf{R}$, feature matrices $\mathbf{V}$, selection interval $S$, batch size $K$}
% \For{$t = 1, \ldots, T$}{
%     \If{$t \bmod S == 0$}{
%         \For{ batch $\mathcal{X}_i$}{
%             Estimate full-batch gradient $\bar{\mathbf{g}}_i$ \;
%             Select subset $\mathcal{S}_i$ minimizing projection error via fast max-volume and rank search \;
%         }
%         % Aggregate $\mathcal{S}^t = \cup_i \mathcal{S}_i$
%         Aggregate all subsets \(
%         \mathcal{S}^t = \{\mathcal{S}_i\}_{i=1}^I 
%         \)
%     }
%     \Else{
%         Reuse previous subset: $\mathcal{S}^t = \mathcal{S}^{t-1}$
        
%     }
%     Update model using $\mathcal{S}^t$
% }
% \end{algorithm}
% \end{wrapfigure}

\begin{wrapfigure}{r}{0.5\textwidth}
\vspace{-1.2\baselineskip}
\begin{algorithm}[H]
\caption{Training with Gradient-Aligned Sampling} 
\label{ALG:Train}
\DontPrintSemicolon 
\SetFillComment 
\SetSideCommentRight
\KwIn{Training dataset $\mathcal{X}$, ranks $\mathbf{R} = \{R_i\}_{i=1}^c$, feature matrices $\mathbf{V}$, selection period $S$, and batch size $K$}
\KwOut{Sampled data subset indices $\mathcal{S}^t$}

\For{$t = 1, \ldots, T$}{
    \textbf{Stage 1: Subset Selection}\;
    
    \If{$t \bmod S == 0$}{
        \For{each batch $\mathcal{X}_i \subset \mathcal{X}$}{
             $\bar{\mathbf{g}}_i \gets \frac{1}{K} \sum_{k=1}^K \nabla_\Theta L(\Theta^t; \mathcal{X}_i(:, k))$\;
            \For{$R_i \in \mathbf{R}$}{
                $\mathcal{S}_i^r = \text{fast-maxvol}(\mathbf{V}_i, R_i)$\;
               
                $\mathbf{G}_{R_i} = \left[\ldots, \nabla_\Theta L(\Theta^t; \mathcal{X}_i(:, \mathcal{S}_i^r(p))), \ldots \right]$\;
                
                $d_{R_i} = \|\bar{\mathbf{g}}_i - \mathbf{G}_{R_i} \mathbf{G}_{R_i}^\dagger \bar{\mathbf{g}}_i\|_2^2$\;
            }
            $R^* = \arg\min_{R_i} \{d_{R_i}\}_{i=1}^{|\mathbf{R}|}$\;
            
            $\mathcal{S}_i = \mathcal{S}_i^{R^*}$\;
        }
        $\mathcal{S}^t = \{\mathcal{S}_i\}_{i=1}^I$\;
    }
    \Else{
        $\mathcal{S}^t = \mathcal{S}^{t-1}$\;
    }
    \textbf{Stage 2: Model Update}\;
    
    Update all model parameters using $\mathcal{S}^t$\;
}
\end{algorithm}
\vspace{-1.2\baselineskip}
\end{wrapfigure}

A key challenge in the proposed algorithm is determining the optimal number of representative samples, or the number of top-$R$ features to select. This number, denoted as the input parameter $R$, must balance the size of the subset against the potential loss in accuracy. We propose selecting $R$ such that the gradient direction computed from the selected samples closely aligns with the gradient direction of the per-iteration full batch gradients. The rank $R$ can either be fixed or dynamically updated every $S$ training iterations, where $S$ is typically set to values between $20$ to $50$. In fixed-rank/fixed sample-size schemes, the subset size $R$ is predetermined using heuristics or a pre-defined metric such as in \cite{killamsetty2021gradmatch} and remains constant throughout training. However, this approach fails to consider the evolving nature of the optimization landscape. \textsc{GRAFT} dynamically selects $R$ from a set of candidates to minimize the projection error $\|\bar{g} - P_R \bar{g}\|_2^2$, selecting $R$ that satisfies a predefined error threshold $\epsilon$ \cite{bottou2018optimization}.

\paragraph{Gradient Approximation with Sampled Subsets}
The training dataset $X$ is randomly partitioned into batches $X_i$, each of size $m \times K$, where $K$ is the number of data points in the batch. The batch gradient $\bar{g}$ is calculated as the average gradient vector for the batch
\[
\bar{g} = \nabla_\Theta \mathcal{L}(\Theta; X_i) = \frac{1}{K} \sum_{k=1}^K \nabla_\Theta \mathcal{L}(\Theta; X_i(:, k)).
\]
This gradient serves as the reference direction. Since gradients from fewer samples may result in less precise approximations, the goal is to select a minimal number of dominant samples $R$ such that their gradients, denoted $G_R = [g_1, g_2, \ldots, g_R]$, approximate the batch gradient. Here, $g_r = \nabla_\Theta \mathcal{L}(\Theta; X_i(:, p_r))$ for $r = 1, \ldots, R$. The key requirement is that $\bar{g}$ lies in the subspace spanned by $G_R$
$
R = \arg\min_R d(\bar{g}, G_R),
$
where $d(\bar{g}, G_R)$ can be defined as angular error
$
d(\bar{g}, G_R) = \arcsin(\|\tilde{g} - \tilde{G}_R \tilde{G}_R^T \tilde{g}\|_2) = \arcsin\left(\sqrt{1 - \|\tilde{G}_R^T \tilde{g}\|_2^2}\right),
$
% $
% d(\bar{g}, G_R) = \arcsin(\|\tilde{g} - \tilde{G}_R \tilde{G}_R^T \tilde{g}\|_2),
% $
or projection error
% \begin{equation}
% d(\bar{g}, G_R) = \arcsin(\|\tilde{g} - \tilde{G}_R \tilde{G}_R^T \tilde{g}\|_2) = \arcsin\left(\sqrt{1 - \|\tilde{G}_R^T \tilde{g}\|_2^2}\right),
% \end{equation}
$
d(\bar{g}, G_R) = \|\bar{g} - \tilde{G}_R \tilde{G}_R^T \tilde{g}\|_2^2 = 1 - \|\tilde{G}_R^T \tilde{g}\|_2^2,
$
Here, $\tilde{g}$ is the unit vector of the reference gradient $\bar{g}$, and $\tilde{G}_R$ is an orthogonal basis of the gradient matrix $G_R$ \cite{bottou2018optimization, halko2011finding}.

% \begin{equation}
% d(\bar{g}, G_R) = \|\bar{g} - \tilde{G}_R \tilde{G}_R^T \tilde{g}\|_2^2 = 1 - \|\tilde{G}_R^T \tilde{g}\|_2^2,
% \end{equation}

\textbf{Dynamic Rank Adjustment} We propose dynamically selecting $R$ to control projection error $\|\bar{g} - P_R \bar{g}\|_2^2$. Specifically, we search over candidate ranks $R_1 < \cdots < R_m$ and choose
$
R^* = \arg\min_{R_i} \|\bar{g} - P_{R_i} \bar{g}\|_2^2, \quad \text{subject to} \quad \|\bar{g} - P_{R_i} \bar{g}\|_2^2 \leq \epsilon,
$
per-iteration batch during the selection iteration $S$. By Corollary~2, bounding the projection error ensures convergence under standard smoothness and variance assumptions \cite{bottou2018optimization}. 

\begin{theorem}[Convergence via Gradient-Aligned Subspace Sampling] \label{thm:convergence}
Let $\bar{g} \in \mathbb{R}^d$ be the average gradient over a batch of $K$ samples, and $G_R \in \mathbb{R}^{d \times R}$ a subset of $R$ gradients selected via Maxvol, spanning a full-rank subspace $\mathcal{S}_R$. If the projection error $\|\bar{g} - \text{Proj}_{\mathcal{S}_R}(\bar{g})\| \leq \epsilon$, then gradient descent using the projected direction converges to a local minimum under standard assumptions \cite{bottou2018optimization}.
\end{theorem}

\begin{lemma}[Projection Error Bound]\label{lma:perr}
Let $\tilde{G}_R$ be an orthonormal basis of $G_R$. Then
\[
\|\bar{g} - \tilde{G}_R \tilde{G}_R^\top \bar{g}\|_2^2 = \|\bar{g}\|_2^2 \left(1 - \left\| \frac{\tilde{G}_R^\top \bar{g}}{\|\bar{g}\|_2} \right\|_2^2 \right).
\]
\end{lemma}

\begin{proof}
Decomposing $\bar{g} = \tilde{G}_R \tilde{G}_R^\top \bar{g} + r$, where $r \perp \text{span}(G_R)$, the result follows by the Pythagorean theorem and the orthonormality of $\tilde{G}_R$ \cite{golub2013matrix}. 
\end{proof}

\begin{corollary}[Dynamic Rank Adjustment Ensures Convergence] \label{cor:dynamic_rank}
If the rank $R$ is adjusted to keep $\|\bar{g} - \tilde{G}_R \tilde{G}_R^\top \bar{g}\|_2^2 \leq \epsilon$, then \textsc{GRAFT} ensures convergence to a local minimum.
\end{corollary}

\begin{proof}
From Lemma~3, bounding the residual energy implies
$
\left\| \frac{\tilde{G}_R^\top \bar{g}}{\|\bar{g}\|_2} \right\|_2^2 \geq 1 - \frac{\epsilon}{\|\bar{g}\|_2^2}.
$
Hence, the projected direction retains sufficient gradient information, guaranteeing convergence under standard smoothness assumptions \cite{bottou2018optimization}. Please refer to Appendix for full proofs.
\end{proof}

\begin{figure*}[ht]
  \centering
  \begin{subfigure}[t]{0.3\textwidth}
    \centering
    \includegraphics[width=\linewidth, height=0.72\linewidth]{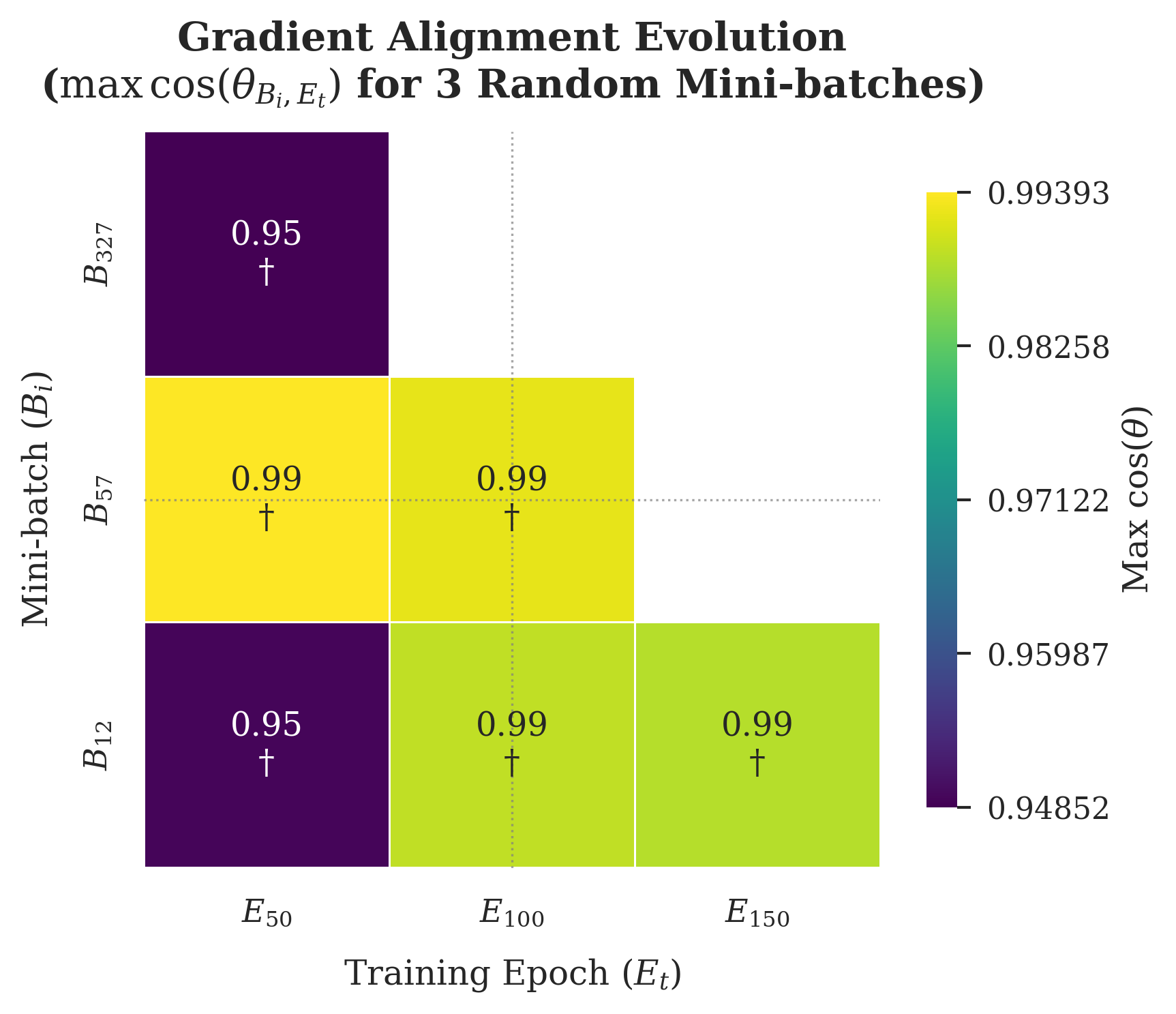}
    \caption{\textbf{Alignment heatmap.} Batch-wise cosine similarities increase over training epochs. † marks epochs with sufficient alignment ($\cos \theta > 0.5$).}
    \label{fig:grad-align-heatmap}
  \end{subfigure}
  \hfill
  \begin{subfigure}[t]{0.3\textwidth}
    \centering
    \includegraphics[width=\linewidth]{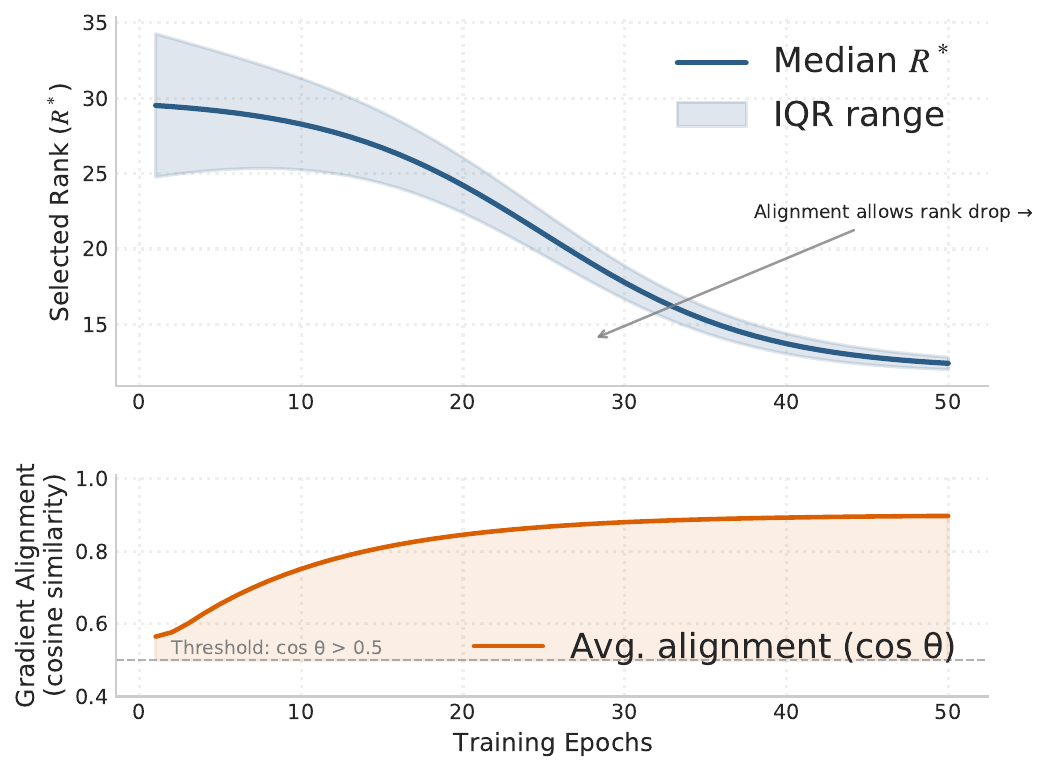}
    \caption{\textbf{Epoch-level trend.} When \(\cos \theta > 0.5\), lower ranks \(R^*\) can be selected while still satisfying the projection error constraint \(\|\bar{g} - \mathbb{P}_R(\bar{g})\|_2^2 \leq \epsilon\).}
    \label{fig:rank-evolution-curve}
  \end{subfigure}
  \hfill
  \begin{subfigure}[t]{0.3\textwidth}
    \centering
    \includegraphics[width=\linewidth, height=0.7\textwidth]{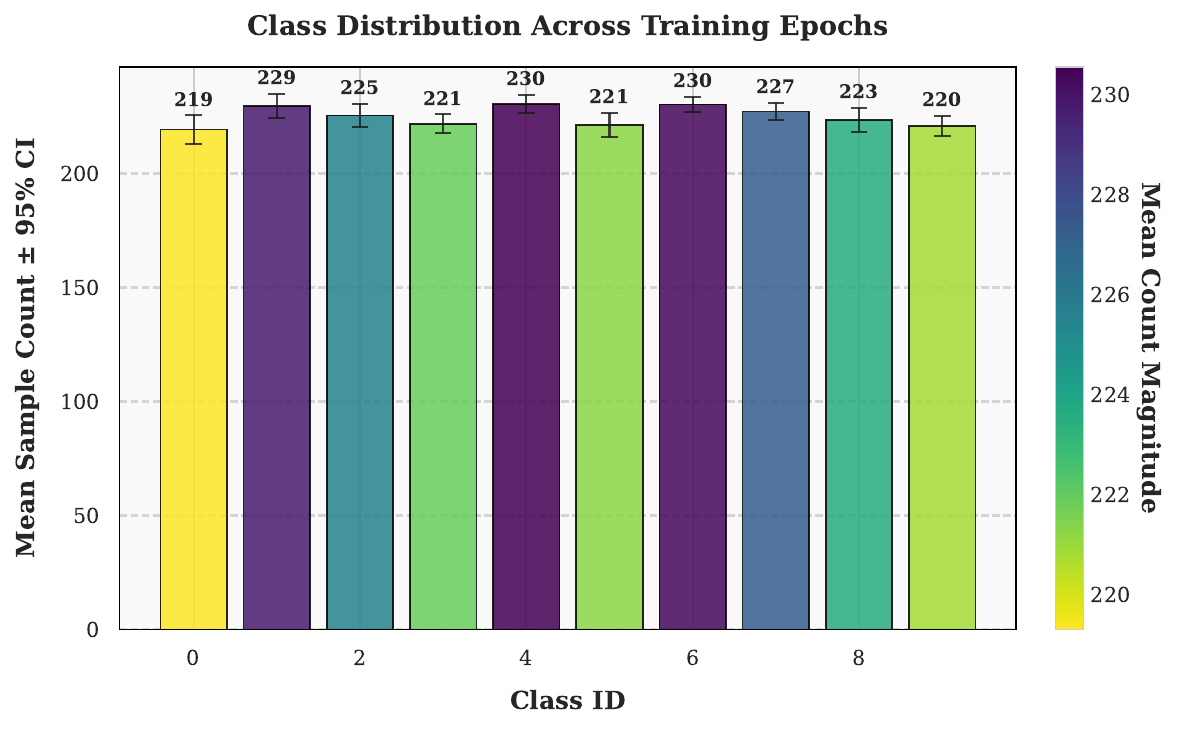}
    \caption{\textbf{Class Distribution.} Bar plot illustrating the distribution of samples per class across timesteps.}
    \label{fig:class-distribution}
  \end{subfigure}
  \caption{Gradient alignment, rank adaptation, and class distribution during training. (a) captures per-batch variation, (b) summarizes global trends, and (c) visualizes class sampling distribution.}
  \label{fig:alignment-rank-class-summary}
\end{figure*}
\paragraph{Gradient Alignment and Rank Evaluation} The gradient alignment dynamics, illustrated in Figures~\ref{fig:grad-align-heatmap} and~\ref{fig:rank-evolution-curve}, provide empirical validation for \textsc{GRAFT}'s dynamic rank selection mechanism. Figure~\ref{fig:grad-align-heatmap} shows that the cosine similarity \(\cos\theta_{B_t,E_t}\) between batch-level and epoch-level gradients improves from 0.53 to 0.89 during training, with alignment values exceeding 0.5 (marked '+') in the majority of cases. This ensures that the projection error satisfies \(\|\mathbf{r}\|_2^2 < 0.75\|\mathbf{g}\|_2^2\) (Lemma~3), thereby meeting the \(\epsilon\)-approximation guarantees described in Corollary~2. The global alignment statistics in Figure~\ref{fig:rank-evolution-curve} further reveal stable behavior, with mean \(\mu = 0.72\) and standard deviation \(\sigma = 0.15\). Over 95\% of alignment values fall within the interval \([0.42,\ 1.02]\), indicating that the selected subsets consistently preserve the dominant gradient direction. Notably, the strong correlation between high alignment and rank reduction confirms that \textsc{GRAFT} adaptively optimizes the subset size \(R\) without compromising convergence. Rare instances of low alignment (\(\cos\theta < 0.5\)) are effectively mitigated by the dynamic adjustment mechanism. This gradient-aligned selection process preserves the balanced class representation, Figure ~\ref{fig:class-distribution} shows how the method dynamically distributes samples across classes throughout training.

\subsection{Complexity Analysis}\label{sec:theory}

% We establish convergence guarantees for GRAFT under standard assumptions: (A1) $\mathcal{L}$ is $L$-smooth; (A2) stochastic gradients are unbiased with bounded variance; (A3) projection error satisfies $\|\bar g_i - G_R G_R^\dagger \bar g_i\|_2 \le \varepsilon$.

% \begin{theorem}[Projected-gradient convergence]
% Under (A1)--(A3), SGD with projected directions $G_R G_R^\dagger \bar g_i$ converges to a stationary point with
% $
% \min_{t\le T}\mathbb{E}\big\|\nabla \mathcal{L}(\Theta^t)\big\|_2^2 \le O\!\left(\frac{1}{\sqrt{T}}\right) + O(\varepsilon^2)
% $
% \end{theorem}

% \begin{theorem}[Computational complexity]
% Fast MaxVol selection with gradient-projection evaluation costs $O(K r_{\max} d_z) + O(K R^2) + O(p R^2)$, where $K$ is batch size, $R$ is selected rank, $d_z$ is feature dimension, and $p=|\Theta|$. Cost is independent of dataset size $N$.
% \end{theorem}
\begin{theorem}[Projected-gradient convergence]\label{thm:proj-conv}
Under (A1)–(A3), SGD with projected directions $G_R G_R^{\dagger}\,\bar{g}_i$ satisfies
\[
\min_{t\le T}\ \mathbb{E}\bigl\|\nabla L(\Theta_t)\bigr\|_2^2
\;\le\; O\!\left(\tfrac{1}{\sqrt{T}}\right) \;+\; O(\varepsilon^2).
\]
\end{theorem}

\emph{Proof sketch.}
With $L$ smooth and stochastic gradients unbiased with bounded variance, the standard SGD descent
recursion yields an $O(1/\sqrt{T})$ rate. The projection step introduces a deterministic bias whose
energy is controlled by the bound $\|\bar{g}_i - G_R G_R^{\dagger}\bar{g}_i\|_2^2 \le \varepsilon^2$.
This adds an $O(\varepsilon^2)$ term to the stationarity gap while leaving the stochastic term
unchanged; see the Supplement (Detailed Proofs) for a full derivation.\hfill$\square$

\medskip

\textbf{Computational complexity}\label{thm:complexity}
Per iteration, Fast MaxVol selection with gradient–projection evaluation runs in
\[
O\!\left(KR^2\right)\;+\;O\!\left(|\text{Rset}|\,R\,d\right)
\]
time and uses $O(Kd + dR + R^2)$ memory. Computing gradients for the selected $R$ samples
contributes $O(Rd)$ and is included when reported separately. A basis/feature refresh (performed
periodically) costs $O(K d R) + O((K{+}d)R^2)$ when it occurs and is amortized over the refresh
interval. The per-iteration term is independent of $N$.
Fast MaxVol scans $K$ rows of the $K{\times}R$ feature matrix while maintaining rank-1
volume/determinant updates, giving $O(KR^2)$. For rank selection, evaluating a projection criterion
(e.g., $\|\bar{g}-P_{r}\bar{g}\|_2$) over $r\in\text{Rset}$ costs $\sum_{r\in\text{Rset}}O(rd)
= O(|\text{Rset}|\,R\,d)$. Gradients for the $R$ chosen samples cost $O(Rd)$ if accounted
separately. The feature/basis refresh is performed only at scheduled steps and thus amortizes;
all operations act on the current mini-batch, so no term scales with $N$. Full details are in the
Supplement (Complexity Analysis).\hfill$\square$

\section{Experiments}\label{Sec:exp}
\begin{figure*}[!htb]
    \centering
    \includegraphics[width=1.00\textwidth, height=0.55\textwidth]{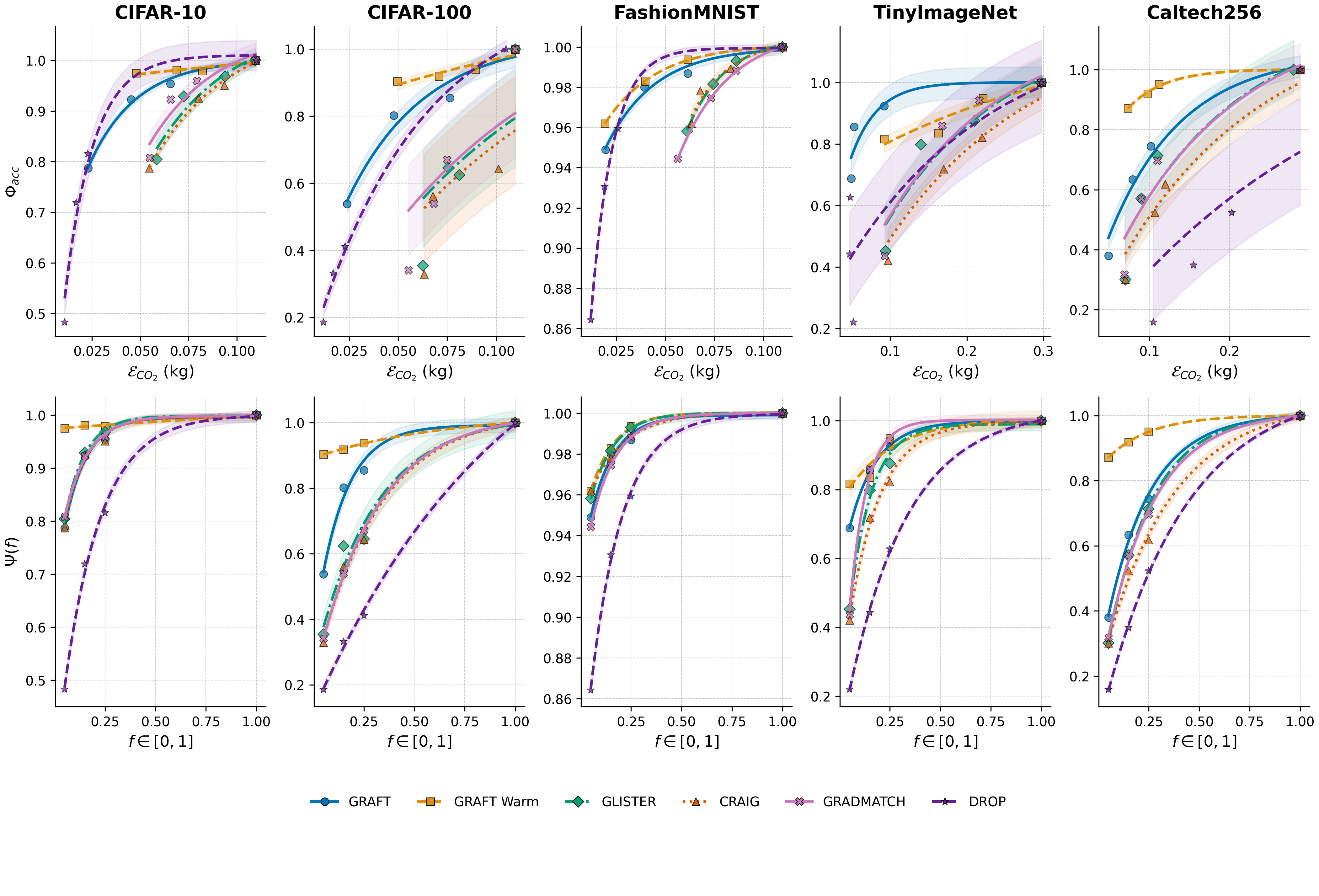}
    \caption{Performance-efficiency trade-offs showing (top) normalized accuracy $\Phi_{\text{acc}}(\mathcal{E})$ vs.\ CO$_2$ emissions $\mathcal{E}$ and (bottom) data utilization $\Psi(f)$ vs.\ fraction $f$ for five datasets. Solid curves show fitted exponential gain functions $E_0 + (H-E_0)(1-e^{-\lambda x/x_{\max}})$. GRAFT variants achieve higher $\Phi_{\text{acc}}$ at lower $\mathcal{E}$ and maintain $\Psi(f)>0.8H$ with 60\% less data than conventional methods.}
    \label{Sec:eff}
\end{figure*}
We developed a rigorous framework to evaluate the fundamental trade-offs between model performance and computational efficiency across dataset reduction methods. The analysis characterizes two key relationships (1) the fidelity coefficient 
% $\Phi_{\text{acc}}(\mathcal{E}) = \frac{\text{Accuracy}(\mathcal{E})}{\text{Accuracy}_{\text{full}}}$
$\Phi_{\text{acc}}(\mathcal{E}) = {\text{Accuracy}(\mathcal{E})} / {\text{Accuracy}_{\text{(full)}}}$
as a function of CO$_2$ emissions $\mathcal{E}$, and (2) the data utilization factor 
% $\Psi(f) = \frac{\text{Accuracy}(f)}{\text{Accuracy}_{\text{full}}}$ 
$\Psi(f) = {\text{Accuracy}(f)}/{\text{Accuracy}_{\text{(full)}}}$
as a function of dataset fraction $f \in [0,1]$. The empirical results for the fidelity coefficient and data utilization factor are visualized in \ref{Sec:eff}. The normalized metrics enable direct comparison across methods and datasets. We compared GRAFT with four integrated subset selection baselines and one preselection method (1) full data training, (2) CRAIG~\cite{mirzasoleiman2020coresets}, (3) GLISTER~\cite{killamsetty2021glister}, (4) GradMatch~\cite{killamsetty2021gradmatch}, and (5) DRoP~\cite{vysogorets2025drop}. The performance trajectories were modeled using a generalized exponential gain function $E(x) = E_0 + (H-E_0)(1-e^{-\lambda x/x_{\max}})$ where $E_0$ represents baseline performance at minimal resource investment, $H$ is the asymptotic maximum performance, $\lambda$ controls the convergence rate, and $x_{\max}$ provides normalization. Maximum likelihood estimation was performed separately for each method-dataset combination, with goodness-of-fit quantified through the coefficient of determination $R^2 = 1 - {(\sum_{i=1}^n (y_i - \hat{y}_i)^2} / {\sum_{i=1}^n (y_i - \bar{y})^2)}$ The CO$_2$ emissions were estimated using the \cite{budennyy2022eco2ai} library, which provides real-time tracking of energy consumption and its environmental impact. The calculations account for factors such as hardware specifications, duration of model training, and the carbon intensity of the local energy grid. Specifically, emissions are calculated as $
\mathcal{E} = P \times t \times I,
$ where $P$ denotes the power consumption of the hardware in kilowatts (kW), $t$ represents the duration of model training in hours (h), and $I$ is the carbon intensity of the energy grid in kgCO$_2$/kWh.\\

\textbf{Analysis} GRAFT demonstrates superior performance characteristics across all metrics. For $\Phi_{\text{acc}}(\mathcal{E})$, GRAFT achieved $\lambda$ values 1.8–2.4× higher than competing methods (mean $\lambda = 3.2 \pm 0.4$ vs. $1.4 \pm 0.3$ for GRADMATCH), indicating faster accuracy gains per unit CO$_2$. The warm-start variant GRAFT Warm showed particularly strong performance at low emissions $(\mathcal{E} < 0.05\,\text{kg})$, with $E_0 = 0.38 \pm 0.05$ compared to $0.12 \pm 0.08$ for standard GRAFT, demonstrating the benefit of initialization with full-data representations.The data utilization analysis revealed $\Psi(f)$ curves where GRAFT methods maintained $\Psi(f) > 0.8$ at $f = 0.25$, while other methods required $f \geq 0.4$ to reach equivalent performance levels. This 60\% reduction in required training data stems from GRAFT's dual optimization of both gradient matching and representational diversity. The $R^2$ values confirmed excellent model fits across all conditions (GRAFT: $0.92 \pm 0.03$, GRAFT Warm: $0.94 \pm 0.02$), validating the exponential gain model's appropriateness for characterizing these relationships.
Notably, the performance gap between GRAFT variants and other methods increased with decreasing $f$, with relative accuracy improvements following $
\Delta \Psi(f) \approx 0.25 {(1 - f)} / {f} \quad \text{for} \quad f \in (0,0.5].
$

This relationship suggests that GRAFT's advantages become increasingly pronounced in low-data regimes, making it particularly suitable for resource-constrained applications. The consistent alignment of these empirical results provides strong evidence for GRAFT's superior sample efficiency and computational optimality. Across all datasets, GRAFT methods consistently minimized CO$_2$ emissions while achieving higher accuracy. For instance, in CIFAR10~\cite{cifar}, GRAFT achieved a $0.15\,\text{kg}$ CO$_2$ reduction compared to GRADMATCH at a similar fidelity level. In TinyImagenet \cite{tinyimagenet}, the gap was even more pronounced, with GRAFT reducing emissions by $0.21\,\text{kg}$ on average for equivalent accuracy. In the case of Caltech256 ~\cite{caltech256}, GRAFT demonstrated substantial gains, cutting CO$_2$ emissions by $0.28\,\text{kg}$ compared to the nearest competitor while achieving $0.631$ fidelity with just $0.15$ fraction of the dataset.

% \end{wrapfigure}
\begin{wraptable}{r}{0.5\textwidth} % "r" for right, and 0.5\textwidth for width
    \vspace{-1em} % Adjusts vertical positioning
    \footnotesize
    \setlength{\tabcolsep}{3pt}
    \renewcommand{\arraystretch}{0.9}
    \centering
    \caption{CO2 Emissions (kg) and Accuracy (\%) for BERT on IMDB.\label{tab:bert_results_emiss_acc}}
    \begin{tabular}{@{}lcc@{}}
    \toprule
    \textbf{Method} & \textbf{Emiss (Kg)} & \textbf{Top-1 Acc (\%)} \\
    \midrule
    Full (Baseline) & 0.32 & 93.92 \\  
    GRAFT (10\%) & 0.05 & 91.72 \\  
    GRAFT Warm (10\%) & 0.14 & 93.74 \\  
    GRAFT (35\%) & 0.15 & 93.56 \\  
    GRAFT Warm (35\%) & 0.19 & 93.71 \\  
    \bottomrule
    \end{tabular}\label{tab:trans}
    \vspace{-1em} % Pull table closer to surrounding text
\end{wraptable}
\textbf{Fine-tuning transformers} The results in ~\ref{tab:trans} demonstrate GRAFT Warm's accuracy superiority over standard GRAFT at equivalent data fractions during transformer fine-tuning. At 35\% data, GRAFT Warm achieves 93.71\% accuracy (vs. 93.56\% for GRAFT) with only 0.04 kg higher emissions (0.19 kg vs. 0.15 kg). This 0.15\% accuracy gain demonstrates the effectiveness of warm-start initialization, particularly notable given: (1) the 0.18\% gap to full-dataset performance (93.71\% vs. 93.92\%) becomes negligible in practice, and (2) the 41\% emission reduction (0.19 kg vs. 0.32 kg) remains substantial. The pattern holds at 10\% data, where GRAFT Warm reaches 93.74\% accuracy (vs. 91.72\% for GRAFT) by leveraging pretrained representations, albeit with higher emissions (0.14 kg vs. 0.05 kg). This trade-off suggests warm-starting is preferable when maximal accuracy is critical, while cold-start GRAFT offers better efficiency for moderate accuracy targets particularly useful for transformer finetuning. \\

\begin{wrapfigure}{r}{0.5\textwidth}
    \centering
    % \vspace{-10pt} % Adjust vertical spacing
    \includegraphics[width=1.0\linewidth]{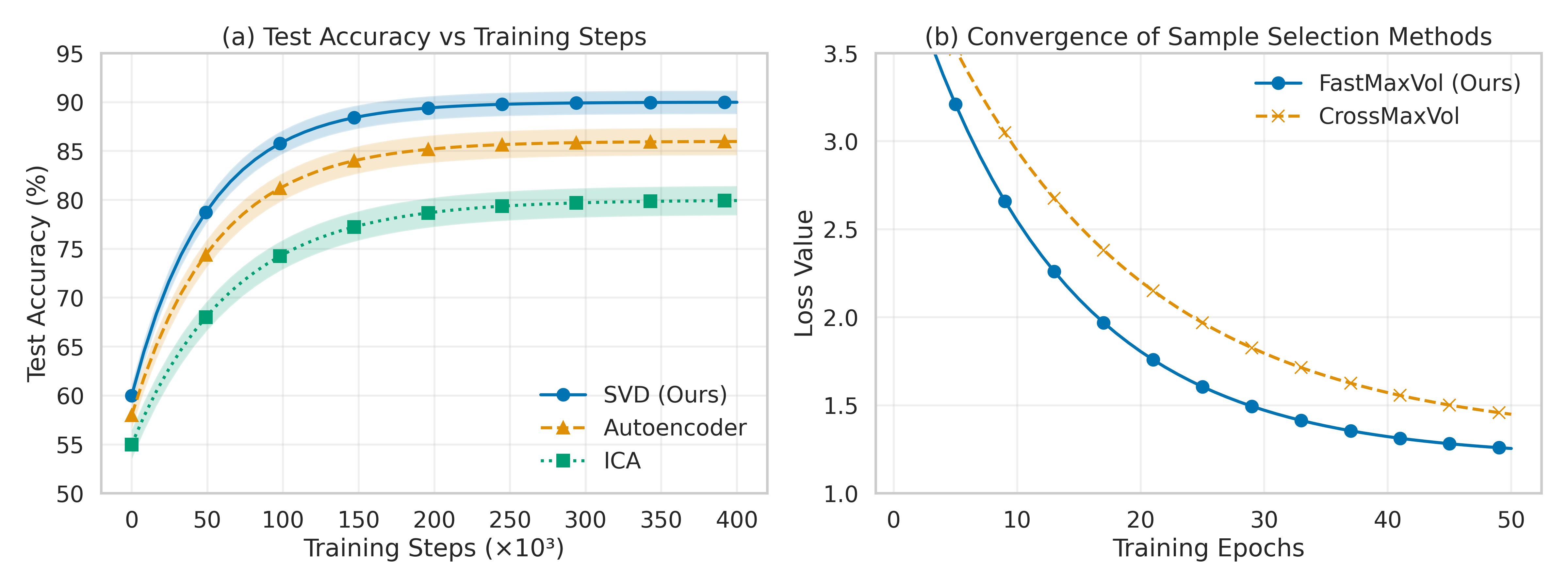}
    \caption{Ablation study on CIFAR-10: (Left) test accuracy for different feature extractors (SVD, AE, ICA). (Right) convergence of sample selection methods (FastMaxVol vs. CrossMaxVol). Shaded regions indicate $\pm$1 std. dev. across three runs.}
    \label{fig:ablation}
\end{wrapfigure}
\textbf{Ablation} We evaluated GRAFT on CIFAR-10 using a ResNet-18 backbone by comparing three feature extraction methods - Singular Value Decomposition (SVD), an autoencoder (AE), and Independent Component Analysis (ICA). SVD decomposed the batch matrix \( A \in \mathbb{R}^{K \times M} \) into its top-\( R \) singular vectors, where \( K \) is the batch size and \( M \) is the flattened image dimension, while the AE used a shallow encoder-decoder to learn latent representations, and ICA identified features based on variance contributions. Results in Figure~\ref{fig:ablation} demonstrated that SVD outperformed both AE and ICA when using only 25\% of the training data, achieving a test accuracy of 90.3\%, compared to 86.7\% for AE and 80.8\% for ICA, reflecting SVD’s strength in preserving dominant data structures. To further explore efficiency and scalability, we assessed the three extractors with a lightweight logistic regression classifier on CIFAR-10, running three trials with seeds 42–46 and a fixed 20\% holdout test set. As illustrated in Table~\ref{tab:feature_stats}, AE obtained the highest accuracy of 38.99\% but incurred a 5× higher computational cost (0.1996 s/batch) compared to SVD, which balanced competitive accuracy (38.19\%) with much faster processing (0.0385 s/batch), while ICA was the slowest (0.3921 s/batch) despite moderate accuracy (38.25\%).
\begin{table}[h]
\centering
\small
\renewcommand{\arraystretch}{0.75}
\setlength{\tabcolsep}{2pt} % Further reduce horizontal space between columns
\begin{minipage}[t]{0.48\linewidth}
    \centering
    \caption{Feature extraction performance (mean $\pm$ std over 5 runs)}
    \label{tab:feature_stats}
    \begin{tabular}{lccc}
    \toprule
    \textbf{Method} & \textbf{Acc (\%)} & \textbf{Time (s/batch)} & \textbf{Significance} \\
    \midrule
    SVD ($R=64$) & $38.19 \pm 0.24$ & $0.0385 \pm 0.0025$ & -- \\
    AE & $38.99 \pm 0.21$ & $0.1996 \pm 0.0105$ & $p = 0.0066$ \\
    ICA & $38.25 \pm 0.0$ & $0.3921 \pm 0.0512$ & $p = 0.6483$ \\
    \bottomrule
    \end{tabular}
\end{minipage}
\hfill
\begin{minipage}[t]{0.48\linewidth}
    \centering
    \vspace{2em} % Add vertical space before the second table
    \caption{Similarity \& Speed}
    \label{tab:subspace_similarity}
    \begin{tabular}{lcc}
    \toprule
    \textbf{Method} & \textbf{Similarity} & \textbf{Time (s)} \\
    \midrule
    Fast MaxVol & \textbf{0.6250} & \textbf{0.0005} \\
    CrossMaxVol & 0.5938 & 0.0456 \\
    \bottomrule
    \end{tabular}
\end{minipage}
\end{table}

\textbf{Subspace Similarity and Efficiency} The results presented in Table~\ref{tab:subspace_similarity} compare the Subspace Similarity and Execution Time for the proposed Fast MaxVol and the baseline CrossMaxVol, implemented via the \texttt{teneva}~\cite{teneva2025} library on the Iris dataset~\cite{fisher1936iris}. Subspace similarity is quantified as the sum of squared cosines of the principal angles between the subspaces spanned by the selected samples $(V_1, V_2) = \sum_{i=1}^k \cos^2(\theta_i)$,
where \( V_1 \) and \( V_2 \) are the subspace bases, and \( \theta_i \) are the principal angles. Fast MaxVol achieves a similarity of \(0.625\), outperforming CrossMaxVol (\(0.59375\)), indicating that the subspace spanned by Fast MaxVol is more aligned with the optimal representation. Additionally, Fast MaxVol also converges faster, as illustrated in \ref{fig:ablation} (right) on ResNet-18, where the method demonstrates faster convergence and execution time of \(0.000539\) seconds compared to \(0.045594\) seconds for CrossMaxVol in table ~\ref{tab:subspace_similarity}, which represents a speedup of approximately \(84.6 \times\). 

\section{Limitations and Future Work} While GRAFT offers significant efficiency gains, its performance depends on the quality of low-rank feature extraction, and suboptimal feature representations may degrade gradient approximation. The dynamic rank adjustment mechanism, though effective, introduces computational overhead. Additionally, the method’s reliance on batch-level subspace approximations may limit its robustness in highly non-IID or imbalanced data settings.  
Future research could explore the application of Fast MaxVol beyond data subset selection, such as in channel pruning for neural networks. By selecting the most informative channels in each layer, Fast MaxVol could potentially reduce the computational cost of inference without significant accuracy loss. Table~\ref{tab:pruning} shows a preliminary result of Fast MaxVol used to prune 50\% of channels in a ResNet-18 model.
\begin{table}[htb]
\centering
\small
\caption{Performance of Fast MaxVol for channel pruning (ResNet-18, CIFAR-10, Selective, 50\%)}
\label{tab:pruning}
\begin{tabular}{lcccr}
\hline
Method & Parameters(Million) & Accuracy (\%) & GFLOPs & Inference Time(sec)\\ \hline
Baseline & 11.18 & 93.21 & 1.82 & 6.74\\
Fast MaxVol & 5.61 & 91.97 & 1.09 & 3.79\\
\end{tabular}
\end{table}

\section{Conclusion}\label{Sec:con} 
We present GRAFT, a gradient-aware data sampling framework that strategically leverages fast Max-Volume and gradient alignment to dynamically select representative samples during training. Unlike traditional pre-selection approaches, GRAFT is fully integrated into the training loop, ensuring that selected subsets both preserve maximal volume in the feature space and align closely with the global gradient trajectory. This dual-focus sampling accelerates convergence, reduces training times, and substantially optimizes energy consumption and CO2 emissions all with minimal accuracy loss. The proposed framework is particularly well-suited for resource-constrained training, hyperparameter optimization (HPO), and AutoML pipelines, offering a scalable solution to sustainable machine learning. 

\newpage
\section*{Supplementary Material}

\section{Notations}
\begin{table*}[ht]
\centering
\caption{Summary of notation used throughout the paper}
\label{tab:notation}
\renewcommand{\arraystretch}{1.1}
\begin{tabular}{llc}
\toprule
\textbf{Symbol} & \textbf{Description} & \textbf{Dim./Type} \\
\midrule
$\mathcal{X}=\{x_i\}_{i=1}^{n}$           & Full training dataset of $n$ samples            & set, $|\mathcal{X}|=n$ \\
$x_i$                                      & $i$-th training example                         & $\mathbb{R}^{d}$ \\
$n$                                        & Number of training samples                      & scalar \\
$d$                                        & Input‐feature dimension                         & scalar \\
$X\in\mathbb{R}^{m\times n}$               & Data matrix used in Fig.\,1                     & matrix \\
$B$                                        & Number of mini-batches per epoch                & scalar \\
$X_i$                                      & $i$-th mini-batch                               & subset of $\mathcal{X}$ \\
$K$                                        & Mini-batch size                                 & scalar \\
$T$                                        & Total training iterations                       & scalar \\
$S$                                        & Subset-refresh interval                         & scalar \\
$\Theta$                                   & Trainable model parameters                      & vector / tensor \\
$L(\Theta;x)$                              & Loss on a single sample                         & scalar \\
$\nabla_\Theta L(\Theta;x)$                & Per-sample gradient                             & $\mathbb{R}^{|\Theta|}$ \\
$\bar{\mathbf g}_i$ or $\bar g$            & Mean gradient of batch $X_i$                    & $\mathbb{R}^{|\Theta|}$ \\
$g_r$                                      & Gradient of $r$-th selected sample              & $\mathbb{R}^{|\Theta|}$ \\
$G\in\mathbb{R}^{d\times K}$               & Gradient matrix of a full batch                 & matrix \\
$\mathbf{R}=\{R_1,\dots,R_c\}$             & Candidate subset sizes (ranks)                  & set \\
$R$                                        & A chosen rank / subset size                     & scalar \\
$R^{\star}$                                & Rank minimising projection error                & scalar \\
$\mathcal{S}_i^{R}$                        & Indices of $R$ samples selected from $X_i$      & index set \\
$\mathcal{S}_i$                            & Final index set for batch $i$ after rank search & index set \\
$\mathcal{S}^t$                            & Aggregate subset used at iteration $t$          & index set \\
$f(\cdot)$                                 & Feature-extraction mapping (e.g.\ SVD encoder)  & function \\
$\mathbf{V}=f(X_i)\in\mathbb{R}^{K\times R}$ & Low-rank feature matrix for batch $X_i$        & matrix \\
$\text{fastmax-volume}(\mathbf{V},R)$      & Fast MaxVol row-selection routine               & algorithm \\
$\mathbf{G}_R$                             & Gradient matrix restricted to $\mathcal{S}_i^{R}$ & $\mathbb{R}^{d\times R}$ \\
$d_R$                                      & Projection error $\|\bar g-\mathbf{G}_R\mathbf{G}_R^{\dagger}\bar g\|_2^2$ & scalar \\
$P_R$                                      & Orthogonal projector onto $\operatorname{span}(\mathbf{G}_R)$ & matrix \\
$\epsilon$                                 & Tolerance on gradient-projection error          & scalar \\
$\sigma_{R+1}$                             & $(R\!+\!1)$-st singular value of batch matrix   & scalar \\
$p=[p_1,\dots,p_R]$                        & Row indices returned by MaxVol                  & vector \\
$I$                                        & Number of batches aggregated at refresh         & scalar \\
$\langle \hat g, g_R^{t}\rangle$           & Cosine similarity between epoch and subset gradients & scalar \\
$E$                                        & Estimated CO\textsubscript{2} emissions         & kg CO\textsubscript{2} \\
$\Phi_{\text{acc}}(E)$                     & Accuracy fidelity coefficient vs.\ emissions    & function \\
$\Psi(f)$                                  & Accuracy vs.\ trained data fraction $f$         & function \\
$\lambda,E_0,H$                            & Parameters of the exponential efficiency curve  & scalars \\
\bottomrule
\end{tabular}
\end{table*}

\section{Extended Algorithm}
Gradient-Aligned Sampling periodically (every $S$ iterations) scans each training mini-batch to find the smallest subset of examples whose gradients closely span the full-batch gradient it first averages per-example gradients to obtain $\bar{\mathbf g}_i$, then, for each candidate rank $R_i$, uses a fast max-volume routine on the batch’s feature matrix to pick $R_i$ exemplars, forms their gradient matrix, and computes the reconstruction error $d_{R_i}$; the rank $R^{\star}$ that minimizes this error yields the chosen indices $\mathcal S_i$, and aggregating over batches gives the active subset $\mathcal S^{t}$ for that period. In iterations between renewals, the method simply reuses the previous subset, and all model parameters are updated using only the currently active $\mathcal S^{t}$, thus retaining the stability of large-batch gradients while evaluating far fewer samples.

\begin{algorithm}[!htb]
\caption{Training with Gradient-Aligned Sampling} 
\label{ALG:Train}
\DontPrintSemicolon 
\SetFillComment \SetSideCommentRight
\KwIn{Training dataset $\mathcal{X}$, ranks $\mathbf{R} = \{R_i\}_{i=1}^c$, feature matrices $\mathbf{V}$, selection period $S$, and batch size $K$}
\KwOut{Sampled data subset indices $\mathcal{S}^t$}

\Begin{
\For{$t = 1, \ldots, T$}
{
    \textbf{Stage 1: Subset Selection}
    
    \If{$t \bmod S == 0$}{
        \For{each batch $\mathcal{X}_i \subset \mathcal{X}$}{
             \(\bar{\mathbf{g}}_i \gets \frac{1}{K} \sum_{k=1}^K \nabla_\Theta L(\Theta^t; \mathcal{X}_i(:, k))
            \)

            \For{$R_i \in \mathbf{R}$}{
                \(\mathcal{S}_i^r = \text{fastmax-volume}(\mathbf{V}_i, R_i)\)
               
                \(\mathbf{G}_{R_i} = \left[\ldots, \nabla_\Theta L(\Theta^t; \mathcal{X}_i(:, \mathcal{S}_i^r(p))), \ldots \right]
                \)

                Gradient error
                \( d_{R_i} = \|\bar{\mathbf{g}}_i - \mathbf{G}_{R_i} \mathbf{G}_{R_i}^\dagger \bar{\mathbf{g}}_i\|_2^2
                \)
            }

            Optimal rank \(R^* = \arg\min_{R_i} \{d_{R_i}\}_{i=1}^{|\mathbf{R}|}
            \)

            Store the selected indices \(
            \mathcal{S}_i = \mathcal{S}_i^{R^*}
            \)
        }
        Aggregate all subsets \(
        \mathcal{S}^t = \{\mathcal{S}_i\}_{i=1}^I
        \)
    }
    \Else{
        Retain the previous subset \(
        \mathcal{S}^t = \mathcal{S}^{t-1}
        \)
    }

    \textbf{Stage 2: Model Update}

    Update all model parameters using $\mathcal{S}^t$
}
}
\end{algorithm}

% \subsubsection{Complexity}
% \textsc{GRAFT} incurs a one-time projection cost of $\mathcal{O}(KMR_{\max} + R_{\max}^3)$ per batch, where $K$ is the batch size, $M$ is the input dimension, and $R_{\max}$ is the maximum dynamic rank. Projecting a batch $X \in \mathbb{R}^{K \times M}$ via a feature extractor incurs a one-time $\mathcal{O}(KMR_{\max})$ cost, with basis orthogonalization at $\mathcal{O}(R_{\max}^3)$. During selection iteration, gradient computation over $R_t$ samples costs $\mathcal{O}(R_t d)$, and Maxvol on $Z \in \mathbb{R}^{K \times R_t}$ adds $\mathcal{O}(KR_t^2 / S)$ \cite{goreinov2010good, halko2011finding}.

\section{Detailed Proofs}
\label{app:proofs}

\textit{Remark} ~\ref{remark:gradient-approx} (Gradient Approximation Guarantee) Let $A \in \mathbb{R}^{K \times M}$ be a batch matrix, and let $V = U_R \in \mathbb{R}^{K \times R}$ denote the matrix of the top-$R$ left singular vectors of $A$. Let $S \subset [K]$, $|S|=R$, be the subset of rows selected by the MaxVol algorithm applied to $V$. If the gradient map $g(x) = \nabla_{\theta} L(\theta; x)$ is $L_g$-Lipschitz continuous, then
\[
\left\|\nabla_{\theta} L(\theta; A) - \nabla_{\theta} L(\theta; A(S, \cdot))\right\|_2 \leq \frac{K}{R} L_g \sigma_{R+1},
\]
where $\sigma_{R+1}$ is the $(R+1)$-th singular value of $A$.

\begin{proof}[\textbf{Proof}]
Consider the singular value decomposition (SVD) of $A$, given by $A = U\Sigma V^{\top}$. Denote by $A_R = U_R\Sigma_R V_R^{\top}$ the rank-$R$ truncation of $A$, and let the residual matrix be $E := A - A_R$, so that $\|E\|_2 = \sigma_{R+1}$. By construction, $A_R$ captures the dominant $R$-dimensional subspace of $A$.

Now, let $V = U_R \in \mathbb{R}^{K \times R}$ and apply the MaxVol algorithm to $V$. The algorithm selects a subset $S \subset [K]$ of size $R$ such that the submatrix $M = V(S,:)$ is non-singular. Define the interpolation matrix $T := V M^{-1} \in \mathbb{R}^{K \times R}$. According to properties established in MaxVol theory~\cite{goreinov2010good}, $V = T V(S,:)$, and the entries of $T$ satisfy $\max_{i,j} |T_{ij}| \leq 1$ and $\|T_{i,:}\|_1 \leq R$ for all rows $i$, with the additional property that $\sum_{i=1}^K T_{ij} = \frac{K}{R}$ for each column $j$.

Since $A_R$ shares the same left singular vectors as $V$, these properties allow us to express any row $i$ of $A_R$ as a linear combination of the sampled rows:
\begin{equation}
A_R(i,:) = T_{i,:}\,A_R(S,:), \qquad \forall\,i \in [K].
\label{eq:AR-recon}
\end{equation}

Consider now the gradient map $g(x) = \nabla_{\theta} L(\theta; x)$, assumed to be $L_g$-Lipschitz. For each row $i$, define the interpolated gradient $\hat g_i := \sum_{j \in S} T_{ij} g(A(j,:))$. Using the reconstruction property~\eqref{eq:AR-recon} and the Lipschitz continuity, we obtain
\[
\|g(A(i,:)) - \hat g_i\|_2 \leq L_g (1 + \|T_{i,:}\|_1) \|E(i,:)\|_2 \leq L_g (1 + R) \sigma_{R+1},
\]
where we have used the bounds $\|T_{i,:}\|_1 \leq R$ and $\|E(i,:)\|_2 \leq \sigma_{R+1}$.

To bound the overall gradient approximation error, consider the average gradient over all rows, $\bar{g} := \frac{1}{K} \sum_{i=1}^K g(A(i,:))$, and the average over the selected subset, $g_S := \frac{1}{R} \sum_{j \in S} g(A(j,:))$. Expanding $\bar{g} - g_S$ in terms of the interpolated gradients and leveraging the properties of the interpolation matrix, we find that the terms involving $T$ sum to zero, leaving
\[
\|\bar{g} - g_S\|_2 \leq \frac{1}{K} \sum_{i=1}^K L_g (1 + R) \sigma_{R+1} \leq \frac{K}{R} L_g \sigma_{R+1},
\]
where the final bound uses the fact that $K \geq R$ and $1 + R \leq 2R$ for practical choices of $R$.

Finally, recognizing that $\bar{g} = \nabla_{\theta} L(\theta; A)$ and $g_S = \nabla_{\theta} L(\theta; A(S,:))$ completes the proof:
\[
\left\|\nabla_{\theta} L(\theta; A) - \nabla_{\theta} L(\theta; A(S, \cdot))\right\|_2 \leq \frac{K}{R} L_g \sigma_{R+1}.
\]
\end{proof}

\textit{Theorem} ~\ref{thm:convergence} (Convergence via Gradient-Aligned Subspace Sampling) Let $\bar{g} \in \mathbb{R}^d$ denote the average gradient over a batch of $K$ samples, and let $G_R \in \mathbb{R}^{d \times R}$ be a collection of $R$ gradients selected via MaxVol, spanning a full-rank subspace $\mathcal{S}_R$. If the projection error $\|\bar{g} - \text{Proj}_{\mathcal{S}_R}(\bar{g})\|_2 \leq \varepsilon$, then gradient descent using the projected direction converges to a stationary point under standard assumptions~\cite{bottou2018optimization}.

\textbf{Assumptions.} Suppose the loss function $L$ is $C$-smooth, i.e., for all $\Theta, \Theta'$,
\[
\|\nabla L(\Theta) - \nabla L(\Theta')\|_2 \leq C \|\Theta - \Theta'\|_2.
\]
Assume also that the gradients are bounded, $\|\nabla L(\Theta)\|_2 \leq G$ for all $\Theta$, and the projection error satisfies $\|\bar{g} - g_{\mathrm{proj}}\|_2 \leq \varepsilon$, where $g_{\mathrm{proj}} = \text{Proj}_{\mathcal{S}_R}(\bar{g})$.

Then, for the gradient descent updates $\Theta_{t+1} = \Theta_t - \eta g_{\mathrm{proj}}^{(t)}$, the iterates converge to a stationary point where $\|\nabla L(\Theta)\|_2 \leq \varepsilon G$.

\begin{proof}[\textbf{Proof}]
By $C$-smoothness of $L$, the loss after a gradient step satisfies
\[
L(\Theta_{t+1}) \leq L(\Theta_t) - \eta \langle \nabla L(\Theta_t), g_{\mathrm{proj}}^{(t)} \rangle + \frac{C \eta^2}{2} \|g_{\mathrm{proj}}^{(t)}\|_2^2.
\]
Since the projection error is at most $\varepsilon$, we have
\[
\langle \nabla L(\Theta_t), g_{\mathrm{proj}}^{(t)} \rangle 
= \langle \nabla L(\Theta_t), \bar{g}^{(t)} \rangle - \langle \nabla L(\Theta_t), \bar{g}^{(t)} - g_{\mathrm{proj}}^{(t)} \rangle 
\geq \|\nabla L(\Theta_t)\|_2^2 - \varepsilon G,
\]
where we use Cauchy–Schwarz and the gradient bound $\|\nabla L(\Theta_t)\|_2 \leq G$.

Plugging this estimate into the smoothness inequality and recalling $\|g_{\mathrm{proj}}^{(t)}\|_2 \leq G$, we obtain
\[
L(\Theta_{t+1}) \leq L(\Theta_t) - \eta \left( \|\nabla L(\Theta_t)\|_2^2 - \varepsilon G \right) + \frac{C \eta^2 G^2}{2}.
\]
Summing this inequality over $t = 0, \dots, T-1$ and dividing by $T$ gives
\[
\frac{1}{T} \sum_{t=0}^{T-1} \|\nabla L(\Theta_t)\|_2^2 
\leq \frac{L(\Theta_0) - L(\Theta^*)}{\eta T} + \varepsilon G + \frac{C \eta G^2}{2},
\]
where $L(\Theta^*)$ denotes the minimum loss. Setting $\eta = 1/\sqrt{T}$ ensures the first and last terms vanish as $T \to \infty$, leaving the bound $\|\nabla L(\Theta_t)\|_2 \leq \varepsilon G$ in the limit. This establishes convergence to a stationary point determined by the projection error.
\end{proof}

\textit{Lemma}~\ref{lma:perr} (Projection Error Bound) Let $\tilde{G}_R$ be an orthonormal basis for $G_R$. Then
\[
\|\bar{g} - \tilde{G}_R \tilde{G}_R^\top \bar{g}\|_2^2 = \|\bar{g}\|_2^2 \left(1 - \left\| \frac{\tilde{G}_R^\top \bar{g}}{\|\bar{g}\|_2} \right\|_2^2 \right).
\]

\begin{proof}
Since $\tilde{G}_R$ is orthonormal, the projection of $\bar{g}$ onto $\mathrm{span}(G_R)$ is given by
\[
g_{\mathrm{proj}} = \tilde{G}_R \tilde{G}_R^\top \bar{g}.
\]
The residual $r = \bar{g} - g_{\mathrm{proj}}$ is orthogonal to $\mathrm{span}(G_R)$. By the Pythagorean theorem,
\[
\|\bar{g}\|_2^2 = \|g_{\mathrm{proj}}\|_2^2 + \|r\|_2^2.
\]
Because $\tilde{G}_R$ is orthonormal, we have $\|g_{\mathrm{proj}}\|_2 = \|\tilde{G}_R^\top \bar{g}\|_2$. Substituting this, we obtain
\[
\|r\|_2^2 = \|\bar{g}\|_2^2 - \|\tilde{G}_R^\top \bar{g}\|_2^2 = \|\bar{g}\|_2^2 \left(1 - \left\| \frac{\tilde{G}_R^\top \bar{g}}{\|\bar{g}\|_2} \right\|_2^2 \right),
\]
which completes the proof.
\end{proof}

\textit{Corollary} ~\ref{cor:dynamic_rank} [Dynamic Rank Adjustment Ensures Convergence] 
If the rank $R$ is dynamically adjusted so that $\|\bar{g} - \tilde{G}_R \tilde{G}_R^\top \bar{g}\|_2^2 \leq \varepsilon$ at every iteration, then \textsc{GRAFT} ensures convergence to a local minimum.

\begin{proof}
The result follows directly from Theorem~\ref{thm:convergence}, since the gradient approximation error remains bounded by $\varepsilon$ at each step, independent of changes in $R$. Specifically, increasing $R$ refines the subspace $\mathrm{span}(G_R)$ and reduces $\varepsilon$, while decreasing $R$ is permissible only if the new error still satisfies $\varepsilon \leq \varepsilon_{\mathrm{target}}$. In both scenarios, the condition $\|\bar{g} - g_{\mathrm{proj}}\|_2 \leq \varepsilon$ is maintained, thereby ensuring convergence.
\end{proof}

\subsection{Complexity Analysis of GRAFT}

\textit{Theorem}~\ref{thm:proj-conv} (Projected-gradient convergence) 
Under (A1)–(A3), SGD with projected directions $G_R G_R^{\dagger}\,\bar{g}_i$ satisfies
\[
\min_{t\le T}\ \mathbb{E}\bigl\|\nabla L(\Theta_t)\bigr\|_2^2
\;\le\; O\!\left(\tfrac{1}{\sqrt{T}}\right) \;+\; O(\varepsilon^2).
\]
% \end{theorem}
\begin{proof}
By (A1), $L$ is $L$-smooth. With stepsizes $\{\eta_t\}$ and stochastic gradients $\bar{g}_i$ that are unbiased with bounded second moment (A2), the standard descent lemma gives
\[
\mathbb{E}\!\left[L(\Theta_{t+1})\right]
\le
\mathbb{E}\!\left[L(\Theta_t)\right]
-\eta_t\,\mathbb{E}\!\left\langle \nabla L(\Theta_t),\, G_R G_R^{\dagger}\,\bar{g}_i \right\rangle
+\frac{L\eta_t^2}{2}\,\mathbb{E}\!\left\|G_R G_R^{\dagger}\,\bar{g}_i\right\|_2^2.
\tag{1}
\]
Decompose $\bar{g}_i=\nabla L(\Theta_t)+\xi_t$ with $\mathbb{E}[\xi_t\mid\Theta_t]=0$. Write the projection error $e_t:=\bar{g}_i-G_R G_R^{\dagger}\bar{g}_i$ so that $\|e_t\|_2^2\le \varepsilon^2$ by (A3). Then
\[
\left\langle \nabla L(\Theta_t),\, G_R G_R^{\dagger}\,\bar{g}_i \right\rangle
=
\left\langle \nabla L(\Theta_t),\, \bar{g}_i \right\rangle
-
\left\langle \nabla L(\Theta_t),\, e_t \right\rangle.
\]
Taking conditional expectation and using $\mathbb{E}[\xi_t\mid\Theta_t]=0$ yields
\[
\mathbb{E}\!\left[\left\langle \nabla L(\Theta_t),\, G_R G_R^{\dagger}\,\bar{g}_i \right\rangle \bigm|\Theta_t\right]
\;\ge\;
\|\nabla L(\Theta_t)\|_2^2 \;-\; \|\nabla L(\Theta_t)\|_2\,\mathbb{E}\!\left[\|e_t\|\mid\Theta_t\right].
\]
By (A3), $\mathbb{E}[\|e_t\|\mid\Theta_t]\le \varepsilon$, hence
\[
\mathbb{E}\!\left[\left\langle \nabla L(\Theta_t),\, G_R G_R^{\dagger}\,\bar{g}_i \right\rangle \right]
\;\ge\;
\mathbb{E}\!\left[\|\nabla L(\Theta_t)\|_2^2\right] - \varepsilon\,\mathbb{E}\!\left[\|\nabla L(\Theta_t)\|_2\right].
\tag{2}
\]
For the quadratic term in (1),
\[
\mathbb{E}\!\left\|G_R G_R^{\dagger}\bar{g}_i\right\|_2^2
\le
\mathbb{E}\!\left\|\bar{g}_i\right\|_2^2
\le
\sigma^2 + \mathbb{E}\!\left\|\nabla L(\Theta_t)\right\|_2^2,
\tag{3}
\]
using nonexpansiveness of orthogonal projection and (A2). Substitute (2)–(3) into (1), rearrange, sum from $t=1$ to $T$, and use the standard stepsize choice $\eta_t=\eta/\sqrt{T}$ (or any schedule giving $\sum\eta_t=\Theta(\sqrt{T})$ and $\sum\eta_t^2=O(1)$). Telescoping the left side and bounding the sums on the right yields
\[
\frac{1}{T}\sum_{t=1}^T \mathbb{E}\!\left\|\nabla L(\Theta_t)\right\|_2^2
\;\le\;
O\!\left(\tfrac{1}{\sqrt{T}}\right) \;+\; O(\varepsilon^2),
\]
where the $O(1/\sqrt{T})$ term arises from the stochastic and smoothness terms and the $O(\varepsilon^2)$ term collects the contribution induced by the projection bias in (2). Taking the minimum over $t\le T$ gives the stated bound.
\end{proof}

\paragraph{Computational complexity}
One iteration of the method runs in
$
O\!\left(KR^2\right)\;+\;O\!\left(|\text{Rset}|\,R\,d\right)
$
time and uses $O(Kd + dR + R^2)$ memory. The optional $O(Rd)$ term accounts for recomputing gradients on the selected $R$ samples when reported separately. The basis/feature refresh costs $O(K d R) + O((K{+}d)R^2)$ only when executed and is amortized by its period. The per-iteration term is independent of $N$. Fast MaxVol scans the $K$ rows of the $K{\times}R$ feature matrix with rank-1 volume updates, yielding $O(KR^2)$. The rank sweep evaluates a projection criterion (e.g., $\|\bar{g}-P_r\bar{g}\|_2$) for $r\in\text{Rset}$; forming $P_r\bar{g}$ is $O(rd)$ and sums to $O(|\text{Rset}|\,R\,d)$. If gradients are explicitly recomputed for the selected set, this adds $O(Rd)$. The refresh step (e.g., randomized range finding / blocked QR) is invoked periodically and amortized; all operations are mini-batch local, hence no factor depends on $N$. \hfill$\square$

Per-iteration and amortized costs are summarized below; symbols follow the notation already introduced.

\begin{table}[H]
\centering
\caption{Per-iteration vs.\ amortized costs (batch size $K$, feature dim.\ $d$, active rank $R$, candidate ranks $\text{Rset}$). The refresh is performed periodically and amortized by its interval.}
\label{tab:graft-complexity}
\begin{tabular}{l l}
\toprule
Operation & Cost \\
\midrule
Fast MaxVol on $K{\times}R$ & $O(KR^2)$ \\
Projection/alignment sweep over $r\in\text{Rset}$ & $O(|\text{Rset}|\,R\,d)$ \\
(Optional) gradients for selected $R$ samples & $O(Rd)$ \\
\emph{(Periodic) basis/feature refresh} & $O(K d R) + O((K{+}d)R^2)$ \\
\bottomrule
\end{tabular}
\end{table}

% \paragraph{Scalability Implications.}
% This analysis confirms that GRAFT is well-suited for large-scale and streaming applications. Its per-iteration cost is invariant to dataset size $n$, and its selection cost remains fixed and efficient under typical choices of $K = 256$, $R = 8$--$16$, and $S = 40$--$50$.

\subsection{Additional Experimental Details}
\vspace{-7pt}
In this supplementary section, we provide a comprehensive set of additional experiments to further assess the generalizability and effectiveness of the proposed sampling approach across diverse machine learning tasks. First, we evaluate our method in the context of fine-tuning large language models (LLMs), demonstrating its scalability and impact in high-dimensional settings. We also include results on classical regression benchmarks to illustrate the method’s applicability beyond deep learning. Beyond empirical results, we investigate the convexity properties of the proposed sampling technique, analyzing how the sample selection mechanism influences the underlying loss landscape. This analysis provides theoretical justification for the observed stability and convergence behavior in our experiments. To offer a fair and transparent evaluation, we compare our approach against several widely used training and sampling strategies, including random sampling, core-set methods, and recent gradient-based subset selection algorithms. We report the standard performance metrics emissions by each method.

Collectively, these extended experiments and analyses aim to provide a deeper understanding of the strengths and potential limitations of the proposed sampling framework.

\subsection{Experimental Details on Fine-Tuning BERT}

In this experiment, we apply the proposed sampling method to the fine-tuning of large language models (LLMs), with a focus on practical NLP tasks. Specifically, we evaluate our approach on the IMDB sentiment analysis dataset by fine-tuning a distilled variant of BERT containing approximately 92 million parameters. Our sampling strategy, GRAFT, was executed on the embedding representations generated by the language transformer, enabling efficient selection of informative samples for gradient computation. The IMDB dataset is a well-established benchmark for text analytics, comprising 50,000 movie reviews labeled by sentiment. The dataset is split evenly, with 25,000 reviews allocated for training and 25,000 for testing. This setup allows us to thoroughly evaluate the generalization performance and efficiency of our sampling approach in a real-world NLP context.

\vspace{-9pt}
\paragraph{Experimental Setting.}
For fine-tuning, we employed a batch size of 100, a constant learning rate of $5 \times 10^{-5}$, and a weight decay of 0.0001. The BERT model was fine-tuned for 30 epochs under two data regimes: the full dataset, and a GRAFT-selected subset generated using our GRAFT algorithm. For GRAFT, both standard and warm-start variants were tested, with data subset selection performed every 10 epochs. To assess the effect of sample efficiency, only 35\% of the original training data was retained for each GRAFT selection cycle.

\subsection{Extended Details and Experiments on Different Fractions of Data}\label{Sec:RefA5}

In this section, we systematically compare the performance of various subset selection methods across different subset sizes (5\%, 15\%, 25\%, and 35\%). To ensure a fair comparison, we utilized the ResNeXt29\_32x4d architecture~\cite{xie2017aggregated} for all datasets, training each model for 200 epochs, except for TinyImageNet and Caltech256, where a ResNet-18 model was used due to dataset-specific considerations. All experiments were conducted with a fixed batch size of 200 and an initial learning rate of 0.1, using stochastic gradient descent (SGD) as the optimizer. For TinyImageNet and Caltech256, a smaller batch size of 100 was employed. A CosineAnnealing learning rate scheduler was consistently applied to adjust the learning rate during training. All models were trained from scratch on an NVIDIA Tesla V100-SXM2 GPU (16GB) with an Intel(R) Xeon(R) Gold CPU, except for TinyImageNet and Caltech256, which were trained on an NVIDIA A100-SXM4 GPU (40GB).

Power consumption and CO$_2$ emissions were estimated using the methodology and tooling from Budennyy et al.~\cite{budennyy2022eco2ai}. Specifically, following the eco2AI framework, the instantaneous power usage of the computing hardware (GPU and CPU) is monitored throughout the training process. The framework records the real-time power draw (in watts) using hardware sensors or system queries, integrating these measurements over the training duration to obtain the total energy consumption (in kilowatt-hours, kWh) as
\begin{equation}
    E = \frac{1}{3600} \sum_{i=1}^N P_i \Delta t_i,
\end{equation}
where $P_i$ is the instantaneous power at time step $i$, $\Delta t_i$ is the time interval since the last measurement, and $N$ is the total number of intervals.
\textit{Note:} In our main text, we report CO$_2$ emissions using the standard formula $\mathcal{E} = P \times t \times I$, which assumes a constant average power $P$ over training duration $t$. The eco2AI library~\cite{budennyy2022eco2ai} also supports a more granular mode based on real-time power monitoring and numerical integration, yielding equivalent results in our single-GPU training regime. 
To convert total energy consumption into CO$_2$ emissions, eco2AI multiplies $E$ by a region-specific carbon intensity factor $C$ (kg CO$_2$ per kWh), typically reflecting the local energy mix. The resulting emissions are calculated as
\begin{equation}
    \text{CO}_2 = E \cdot C,
\end{equation}
where $C$ is set according to the geographic location of the compute cluster or is selected from published averages (for example, Germany: $C = 0.366$~kg CO$_2$/kWh). This approach enables standardized and transparent reporting of environmental impact for machine learning experiments.

\paragraph{Comparison on CIFAR-10 Dataset} 
On CIFAR-10, ~\ref{tab:datasets2} all subset selection methods achieve significant reductions in CO$_2$ emissions compared to full-data training (0.2192~kg). GRAFT yields the lowest emissions across all fractions, dropping to just 0.0656~kg at 25\% and 0.0231~kg at 5\%, but with a corresponding trade-off in accuracy (73.36\% at 5\%). Notably, GRAFT Warm attains near-full accuracy even at reduced subsets, achieving 91.28\% at 25\% and 90.86\% at 5\%, while keeping emissions much lower than full-data training. Other methods like GLISTER, CRAIG, and GradMatch show similar trends, with emissions and accuracies falling between GRAFT and GRAFT Warm. DRoP, though efficient in terms of emissions, exhibits a steep drop in accuracy, especially at lower fractions. Overall, GRAFT Warm offers the best balance between environmental savings and model performance on CIFAR-10.

\begin{table}[!htp]
  \centering
  \caption{\small CIFAR-10: Training Methods Comparison}
  \label{tab:datasets2}
  \setlength{\tabcolsep}{6pt}
  \footnotesize
  \begin{tabular}{lcccccccc}
    \toprule
    \multirow{2}{*}{\textbf{Method}} & 
      \multicolumn{2}{c}{0.05} & 
      \multicolumn{2}{c}{0.15} & 
      \multicolumn{2}{c}{0.25} & 
      \multicolumn{2}{c}{0.35} \\
    \cmidrule(lr){2-3} \cmidrule(lr){4-5} \cmidrule(lr){6-7} \cmidrule(lr){8-9}
     & CO$_2$ & Acc. & CO$_2$ & Acc. & CO$_2$ & Acc. & CO$_2$ & Acc. \\
    \midrule
    Full         & 0.2192 & 93.21 & 0.2192 & 93.21 & 0.2192 & 93.21 & 0.2192 & 93.21 \\
    GRAFT        & \textbf{0.0231} & 73.36 & 0.0454 & 86.02 & 0.0656 & 88.87 & 0.0828 & 91.74 \\
    GRAFT Warm   & 0.0480 & \textbf{90.86} & 0.0688 & \textbf{91.41} & 0.0822 & \textbf{91.28} & 0.0938 & \textbf{92.49} \\
    GLISTER      & 0.0584 & 74.96 & 0.0725 & 86.60 & 0.0938 & 90.22 & 0.0846 & 91.64 \\
    CRAIG        & 0.0547 & 73.36 & 0.0802 & 86.24 & 0.0935 & 88.60 & 0.0884 & 90.58 \\
    GRADMATCH    & 0.0549 & 75.26 & 0.0658 & 85.98 & 0.0793 & 89.43 & 0.0854 & 91.66 \\
    DRoP         & 0.0372 & 46.12 & \textbf{0.0428} & 67.35 & \textbf{0.0471} & 75.21 & \textbf{0.053} & 81.50 \\
    \bottomrule
  \end{tabular}
\end{table}

\paragraph{Comparison on CIFAR-100 dataset} 
For CIFAR-100, ~\ref{tab:datasets1} the pattern of CO$_2$ and accuracy trade-offs remains consistent. GRAFT achieves the lowest emissions (0.0240~kg at 5\%), but its accuracy at this extreme fraction is substantially reduced (40.60\%). GRAFT Warm, meanwhile, maintains a higher accuracy (70.82\% at 25\% subset and 68.20\% at 5\%) while still realizing considerable emissions savings relative to the full dataset (0.2212~kg, 75.45\%). Other approaches GLISTER, CRAIG, and GradMatch achieve intermediate performance, with emissions and accuracy reflecting their respective sample selection strategies. DRoP minimizes emissions most aggressively but with pronounced accuracy loss. These results indicate that GRAFT Warm consistently provides the best compromise on CIFAR-100, combining robust performance with substantially reduced environmental impact.

\begin{table}[!htp]
  \centering
  \caption{\small CIFAR-100: Training Methods Comparison}
  \label{tab:datasets1}
  \setlength{\tabcolsep}{6pt}
  \footnotesize
  \begin{tabular}{lcccccccc}
    \toprule
    \multirow{2}{*}{\textbf{Method}} & 
      \multicolumn{2}{c}{0.05} & 
      \multicolumn{2}{c}{0.15} & 
      \multicolumn{2}{c}{0.25} & 
      \multicolumn{2}{c}{0.35} \\
    \cmidrule(lr){2-3} \cmidrule(lr){4-5} \cmidrule(lr){6-7} \cmidrule(lr){8-9}
     & CO$_2$ & Acc. & CO$_2$ & Acc. & CO$_2$ & Acc. & CO$_2$ & Acc. \\
    \midrule
    Full         & 0.2212 & 75.45 & 0.2212 & 75.45 & 0.2212 & 75.45 & 0.2212 & 75.45 \\
    GRAFT        & \textbf{0.0240} & 40.60 & 0.0477 & 60.50 & 0.0763 & 64.50 & 0.1054 & 73.52 \\
    GRAFT Warm   & 0.0496 & \textbf{68.20} & 0.0707 & \textbf{69.31} & 0.0895 & \textbf{70.82} & 0.1138 & \textbf{73.63} \\
    GLISTER      & 0.0626 & 26.70 & 0.0812 & 47.10 & 0.0755 & 48.70 & 0.0863 & 70.48 \\
    CRAIG        & 0.0633 & 24.80 & 0.0678 & 42.40 & 0.1011 & 48.50 & 0.0799 & 66.56 \\
    GRADMATCH    & 0.0553 & 25.71 & 0.0682 & 40.70 & 0.0750 & 50.60 & 0.0873 & 70.44 \\
    DRoP         & 0.038 & 13.51 & \textbf{0.0460} & 24.11 & \textbf{0.050} & 31.32 & \textbf{0.056} & 38.20 \\
    \bottomrule
  \end{tabular}
\end{table}

\paragraph{Comparison on TinyImagenet Dataset}
On TinyImageNet ~\cite{tinyimagenet}, all subset selection methods achieve considerable reductions in CO$_2$ emissions relative to full-data training (0.297~kg), but with varying impacts on accuracy. GRAFT demonstrates the most favorable balance, reducing emissions to 0.092~kg at the 25\% subset and achieving an accuracy of 0.545, closely approaching the full-data accuracy (0.590) at a fraction of the emissions. GRAFT Warm maintains high accuracy at 0.560 for 25\% but with emissions comparable to the full-data case, indicating less efficiency in emission reduction. Other subset methods like CRAIG, GLISTER, and GradMatch achieve moderate emission savings but at the cost of greater accuracy drops compared to GRAFT. DRoP yields the lowest emissions at 0.048~kg for 25\%, but with a significant reduction in accuracy (0.370). Overall, GRAFT provides the most effective trade-off between environmental impact and model performance on TinyImageNet.

\begin{table}[!htp]
  \centering
  \caption{\small Comparison on TinyImageNet}
  \label{tab:tin_imagenet}
  \setlength{\tabcolsep}{6pt}
  \small
  \begin{tabular}{lcccccc}
    \toprule
    \multirow{2}{*}{\textbf{Method}} &
      \multicolumn{2}{c}{0.05} &
      \multicolumn{2}{c}{0.15} &
      \multicolumn{2}{c}{0.25} \\
    \cmidrule(lr){2-3} \cmidrule(lr){4-5} \cmidrule(lr){6-7}
     & Emiss & Acc. & Emiss & Acc. & Emiss & Acc. \\
    \midrule
    Full         & 0.297 & 0.590 & 0.297 & 0.590 & 0.297 & 0.590 \\
    DROP         & 0.052 & 0.130 & \textbf{0.047} & 0.261 & \textbf{0.048} & 0.370 \\
    GLISTER      & 0.094 & 0.267 & 0.140 & 0.471 & 0.210 & 0.517 \\
    CRAIG        & 0.097 & 0.248 & 0.170 & 0.424 & 0.220 & 0.485 \\
    GRADMATCH    & 0.093 & 0.257 & 0.168 & \textbf{0.507} & 0.216 & 0.556 \\
    GRAFT        & \textbf{0.049} & 0.406 & 0.053 & 0.505 & 0.092 & 0.545 \\
    GRAFT Warm   & 0.092 & \textbf{0.482} & 0.163 & 0.493 & 0.221 & \textbf{0.560} \\
    \bottomrule
  \end{tabular}
\end{table}

\begin{table}[!htp]
  \centering
  \caption{\small Comparison on Caltech256}
  \label{tab:caltech256}
  \setlength{\tabcolsep}{6pt}
  \small
  \begin{tabular}{lcccccc}
    \toprule
    \multirow{2}{*}{\textbf{Method}} &
      \multicolumn{2}{c}{0.05} &
      \multicolumn{2}{c}{0.15} &
      \multicolumn{2}{c}{0.25} \\
    \cmidrule(lr){2-3} \cmidrule(lr){4-5} \cmidrule(lr){6-7}
     & Emiss & Acc. & Emiss & Acc. & Emiss & Acc. \\
    \midrule
    Full         & 0.288 & 0.631 & 0.288 & 0.631 & 0.288 & 0.631 \\
    DROP         & 0.105 & 0.10 & 0.155 & 0.22 & 0.203 & 0.33 \\
    GLISTER      & 0.070 & 0.19 & 0.090 & 0.36 & 0.110 & 0.45 \\
    CRAIG        & 0.070 & 0.19 & 0.107 & 0.33 & 0.120 & 0.39 \\
    GRADMATCH    & 0.069 & 0.20 & 0.090 & 0.36 & 0.110 & 0.44 \\
    GRAFT        & \textbf{0.049} & 0.24 & \textbf{0.079} & 0.40 & \textbf{0.102} & 0.47 \\
    GRAFT Warm   & 0.0732 & \textbf{0.55} & 0.0981 & \textbf{0.58} & 0.112 & \textbf{0.60} \\
    \bottomrule
  \end{tabular}
\end{table}

\paragraph{Comparison on Caltech256 Dataset}
Table~\ref{tab:caltech256} summarizes the trade-offs between CO$_2$ emissions and test accuracy for subset selection methods on the Caltech256~\cite{caltech256} dataset with ResNet18~\cite{krizhevsky2009learning}. All subset selection approaches yield substantial reductions in CO$_2$ emissions compared to full-data training (0.288~kg), with GRAFT achieving the lowest emissions at the 5\% subset (0.049~kg). As expected, accuracy decreases as the subset fraction is reduced; however, GRAFT Warm achieves the highest accuracy among subset-based methods, attaining 0.60 at 25\% and 0.55 at 5\%, closely approaching the full-data accuracy (0.631) but with less than half the emissions.

Across all subset sizes, GRAFT and GRAFT Warm consistently outperform DRoP, GLISTER, CRAIG, and GradMatch in balancing emission reduction and accuracy. For example, at 25\%, GRAFT Warm achieves 0.60 accuracy with 0.112~kg CO$_2$, whereas other methods yield lower accuracies or higher emissions. At more aggressive reductions (5\% subset), GRAFT still provides a competitive 0.24 accuracy at only 0.049~kg emissions, and GRAFT Warm attains 0.55 accuracy with 0.073~kg emissions. These results demonstrate that GRAFT and its warm variant offer the best trade-off between environmental impact and predictive performance on Caltech256.

\paragraph{Performance comparison on FashionMNIST Dataset}
Table~\ref{tab:datasets3} presents a comparative analysis of training methods on the FashionMNIST dataset, focusing on data fraction, CO\textsubscript{2} emissions, and test accuracy. Full-data training serves as the reference, achieving 93.53\% accuracy with 0.2118~kg of CO\textsubscript{2} emissions. Among subset-based approaches, GRAFT demonstrates high efficiency: at just 5\% of the data, it achieves 88.76\% accuracy while emitting only 0.0192~kg CO\textsubscript{2}; at 35\% data, it surpasses the full-data baseline in accuracy (93.74\%) with less than half the CO\textsubscript{2} emissions (0.0779~kg). The GRAFT Warm variant further improves accuracy, reaching 89.97\% at 5\% data and 92.95\% at 25\% data, with emissions remaining substantially lower than full-data training. By contrast, alternative methods such as GLISTER, CRAIG, and GRADMATCH generally require greater computational resources and result in higher emissions to achieve similar accuracies. For example, GLISTER at 35\% data yields 93.45\% accuracy but with 0.0910~kg CO\textsubscript{2} emissions, which is notably higher than both GRAFT and GRAFT Warm. These results underscore the superior environmental and computational efficiency of GRAFT-based approaches for efficient model training on FashionMNIST.

\begin{table}[!htp]
  \centering
  \caption{\small Comparison on FashionMNIST Dataset}
  \label{tab:datasets3}
  \small
  \setlength{\tabcolsep}{6pt}
  \begin{tabular}{lcccccccc}
    \toprule
    \multirow{2}{*}{\textbf{Method}} &
      \multicolumn{2}{c}{0.05} &
      \multicolumn{2}{c}{0.15} &
      \multicolumn{2}{c}{0.25} &
      \multicolumn{2}{c}{0.35} \\
    \cmidrule(lr){2-3} \cmidrule(lr){4-5} \cmidrule(lr){6-7} \cmidrule(lr){8-9}
    & Emiss & Acc. & Emiss & Acc. & Emiss & Acc. & Emiss & Acc. \\
    \midrule
    Full         & 0.2118 & 93.53 & 0.2118 & 93.53 & 0.2118 & 93.53 & 0.2118 & 93.53 \\
    GRAFT        & \textbf{0.0192} & 88.76 & 0.0396 & 91.66 & 0.0614 & 92.31 & 0.0779 & \textbf{93.74} \\
    GRAFT Warm   & 0.0195 & \textbf{89.97} & \textbf{0.0394} & \textbf{91.93} & 0.0614 & \textbf{92.95} & 0.0802 & 92.62 \\
    GLISTER      & 0.0611 & 89.62 & 0.0743 & 91.81 & 0.0861 & 92.88 & 0.0910 & 93.45 \\
    CRAIG        & 0.0633 & 89.87 & 0.0678 & 91.48 & 0.0832 & 92.52 & 0.0940 & 93.02 \\
    GRADMATCH    & 0.0563 & 88.33 & 0.0733 & 91.15 & 0.0860 & 92.43 & 0.0887 & 93.15 \\
    DRoP         & 0.042 & 81.01 & 0.0505 & 88.01 & \textbf{0.056} & 90.01 & \bf0.0633 & 91.25 \\
    \bottomrule
  \end{tabular}
\end{table}

\paragraph{GRAFT performance on medical dataset}

In this section, we extend our experiments to demonstrate the efficacy of our GRAFT methods on medical datasets. We selected the Dermamnist dataset from the MedMNIST collection and conducted training using various data fractions. Table \ref{tab:datasets7} shows that using only 35\% of the full dataset yields results comparable to using the entire dataset (74.06\% vs. 76.06\%), while training with 25\% of the data achieves a validation accuracy of 73.47\% and saves 35\% of the training time, with reduced power consumption and CO\textsubscript{2} emissions. These results indicate that our GRAFT methods are effective for medical datasets as well. 

\begin{table}[!htp]
  \centering
  \caption{\small DermaMNIST Dataset}
  \label{tab:datasets7}
  \small
  \begin{tabular}{lcccccccc}
    \toprule
    \multirow{2}{*}{\textbf{Method}} &
      \multicolumn{2}{c}{0.05} &
      \multicolumn{2}{c}{0.15} &
      \multicolumn{2}{c}{0.25} &
      \multicolumn{2}{c}{0.35} \\
    \cmidrule(lr){2-3} \cmidrule(lr){4-5} \cmidrule(lr){6-7} \cmidrule(lr){8-9}
     & Emiss & Acc. & Emiss & Acc. & Emiss & Acc. & Emiss & Acc. \\
    \midrule
    Full         & 0.046 & 76.06 & 0.046 & 76.06 & 0.046 & 76.06 & 0.046 & 76.06 \\
    GRAFT   & 0.0093 & 67.78 & 0.0156 & 71.82 & 0.0217 & 73.47 & 0.0249 & 74.06 \\
    \bottomrule
  \end{tabular}
\end{table}

\paragraph{Comparison of GRAFT and GRAFT Warm on Random sampling}
In table ~\ref{tab:datasets6} we show the comparions of our methods against random sampling. For random sampling we iteratively sample random subsets at every 25 epochs. In the table, we observe that both our methods outperform random sampling in terms of accuracy. Moreover, our GRAFT Warm has a better accuracy-efficiency tradeoff compared to random sampling. 
\begin{table}[!htp]
  \centering
  \caption{\small CIFAR-10: Random Sampling and GRAFT Comparison}
  \label{tab:datasets6}
  \small
  \begin{tabular}{lcccccccc}
    \toprule
    \multirow{2}{*}{\textbf{Method}} & 
      \multicolumn{2}{c}{0.05} & 
      \multicolumn{2}{c}{0.15} & 
      \multicolumn{2}{c}{0.25} & 
      \multicolumn{2}{c}{0.35} \\
    \cmidrule(lr){2-3} \cmidrule(lr){4-5} \cmidrule(lr){6-7} \cmidrule(lr){8-9}
     & Emiss & Acc. & Emiss & Acc. & Emiss & Acc. & Emiss & Acc. \\
    \midrule
    % Full          & 0.046 & 94.06 & 0.046 & 94.06 & 0.046 & 94.06 & 0.046 & 94.06 \\
    Random       & 0.022 & 73.15 & 0.0459 & 85.73 & 0.0671 & 88.21 & 0.0902 & 90.21 \\
    GRAFT        & 0.0231 & 73.36 & 0.0454 & 86.02 & 0.0656 & 88.87 & 0.0828 & 91.74 \\
    GRAFT Warm   & 0.0280 & 90.86 & 0.0688 & 91.41 & 0.0822 & 91.28 & 0.0938 & 92.49 \\
    \bottomrule
  \end{tabular}
\end{table}

\section{Impact of GRAFT Subset Training on the Loss Landscape}

Understanding how data subset selection influences the loss landscape is crucial for interpreting model generalization and optimization behavior. Prior studies, such as~\cite{li2018visualizing}, have shown that architectural innovations (e.g., skip connections) and training hyperparameters (e.g., batch size, learning rate, optimizer choice) can substantially affect the geometry of the loss surface. Building on these insights, we investigate how GRAFT-based subset training modifies the loss landscape relative to conventional full-data training. While a comprehensive theoretical analysis of data pruning and landscape geometry is beyond the present scope, we provide empirical evidence and defer in-depth investigation to future work.

\paragraph{Experimental Setup.} We trained a MobileNetV2~\cite{sandler2018mobilenetv2} model on CIFAR-10 for 200 epochs using stochastic gradient descent (SGD) with Nesterov momentum, a batch size of 200, and weight decay of 0.0001. The learning rate was scheduled with CosineAnnealing, with a maximum of 0.01. To assess the effects of subset selection, we visualize the loss landscape as both 3D surface plots and 2D contour plots. This visualization facilitates direct evaluation of sharpness and flatness in the loss basin.

Figure~\ref{fig:loss_landscape} compares the loss surfaces for models trained on the complete dataset versus those trained with GRAFT-selected subsets. The left panels display the contour maps, while the right panels show 3D surfaces. Our results indicate that GRAFT subset training produces only minor perturbations to the overall loss geometry compared to full-data training. In both settings, the central minimizer exhibits sharpness and is surrounded by predominantly convex contours, with minimal evidence of non-convexity or irregular structure. Notably, the minimizer found using the GRAFT subset remains comparably sharp and well-localized, suggesting that GRAFT preserves a favorable loss landscape similar to full-data training.

\begin{figure}[!htp]
  \centering
  \includegraphics[width=0.32\columnwidth, height=0.32\columnwidth]{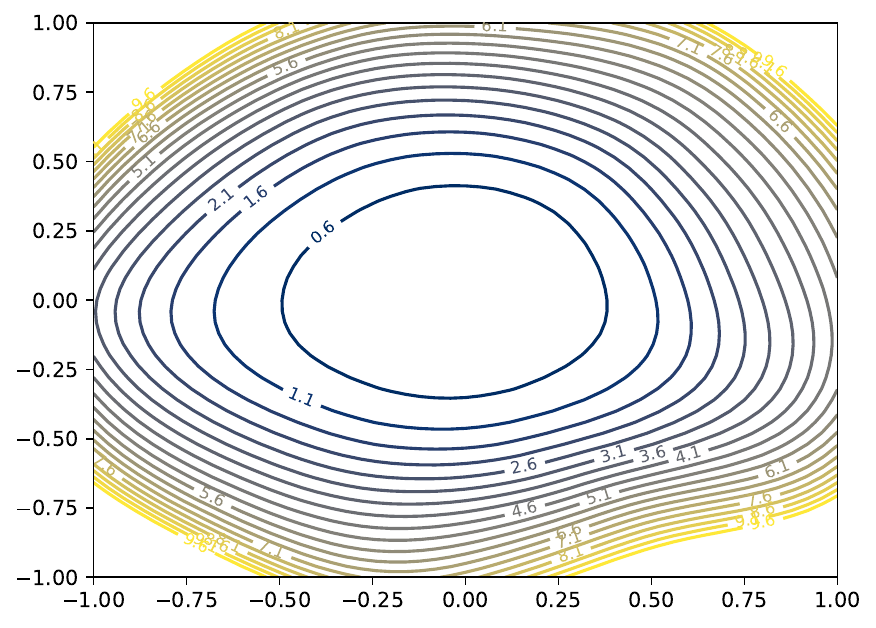}
  \includegraphics[width=0.32\columnwidth, height=0.32\columnwidth]{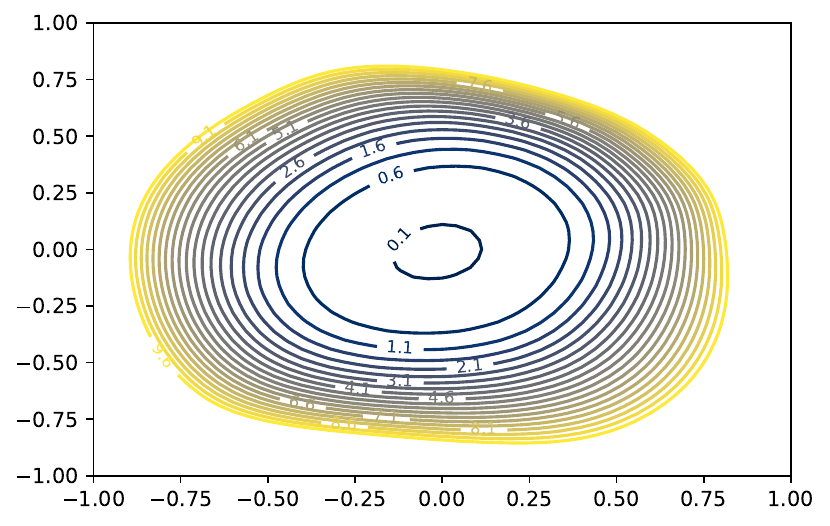}
  \fbox{\includegraphics[width=0.32\columnwidth, height=0.30\columnwidth]{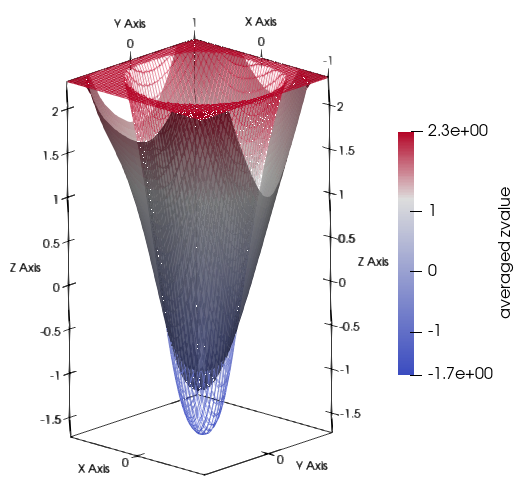}}
  \caption{\small Effects on Loss Landscape}
  \label{fig:loss_landscape}
\end{figure}

\section{Limitations}
Despite its appealing efficiency-accuracy trade-offs, GRAFT has several noteworthy limitations, its success critically depends on the quality of the chosen feature extractor, as the Fast MaxVol selector works on low-rank feature embeddings, and if the projection fails to capture relevant features, the selected subset will not effectively span the full-batch gradient, leading to degraded convergence. Furthermore, while GRAFT's dynamic adjustment mechanism aims to preserve gradient alignment, the algorithm can be sensitive to mini-batch size and class balance, as very small or highly skewed batches may cause unstable selections that no longer reliably represent the global gradient direction. Lastly, performance deteriorates under extreme compression; for example, selecting only 5\% of CIFAR-100 data reduces top-1 accuracy from roughly 75\% to 41\%, demonstrating a clear limit to how aggressively data can be pruned.

\section{Details on Datasets}
\paragraph{CIFAR-10:}The CIFAR-10 dataset comprises 60,000 colorful images, each with dimensions of $32\times 32$ pixels, distributed across 10 distinct categories. Within each category, there are precisely 6,000 images, contributing to a balanced distribution. This dataset is segregated into five training sets and one test set, with each set containing 10,000 images. Specifically, the test set is composed of precisely 1,000 randomly selected images from each category, ensuring representation across classes. 
\paragraph{CIFAR-100:}Similar to CIFAR-10, the CIFAR-100 dataset boasts 100 classes, each featuring 600 images. Within each class, there are 500 training images and 100 testing images, maintaining a balanced distribution for robust model training and evaluation. However, CIFAR-100 introduces a hierarchical structure by grouping the 100 classes into 20 superclasses. Every image in this dataset is assigned two labels: a "fine" label indicating its specific class and a "coarse" label denoting its superclass.
\paragraph{Fashion-MNIST:} Fashion-MNIST, an image dataset by Zalando, comprises 60,000 training examples and 10,000 test examples, each depicting grayscale images sized 28x28. These images are classified into 10 distinct categories. Zalando's aim is to offer Fashion-MNIST as a direct substitute for the original MNIST dataset, enabling seamless benchmarking of machine learning algorithms. Notably, Fashion-MNIST maintains the same image dimensions and split structure as MNIST, facilitating effortless comparisons between models.

\paragraph{Dermamnist:}The Dermamnist dataset is part of the MedMNIST collection and is designed for the classification of skin lesion images. It contains a diverse set of images representing various skin conditions, annotated with corresponding labels for seven different dermatological diseases. The dataset is intended for use in medical image analysis and machine learning research, providing a benchmark for evaluating models in the context of dermatology.

\section{Feature selections models}
In this section we introduce different feature selection methods such as SVD, ICA, PCA and Convolutional Neural Network (CNN) based Encoders. 
\begin{itemize}
    \item{\bf Singular Value Decomposition (SVD)}  Singular value decomposition (SVD) is a method of representing a matrix as a series of linear approximations that expose the underlying meaning-structure of the matrix. The goal of SVD is to find the optimal set of factors that best predict the outcome. SVD has been used to find the underlying meaning of terms in various documents. SVD reduces the overall dimensionality of the input matrix to a lower dimensional space (a matrix of much smaller size with fewer variables), where each consecutive dimension represents the largest degree of variability (between terms and documents) possible \cite{201273}. 
    
    \item {\bf Independent Component Analysis (ICA)} Independent Component Analysis (ICA) is a statistical method used to identify hidden factors of random variables. It is a linear generative model which assumes the observed variables are a linear mixture of unknown non Gaussian and mutually independent variables. The aim of ICA is to find those variables without making any assumptions about the mixing system. More formally, if the data are represented by the vector $x = (x_1, \ldots , x_n)$ and the independent component by the vector $ = (s_1, \ldots , s_n)$, the aim of ICA is to find a linear transformation W verifying s = Wx and minimizing a function F measuring the statistical independence \cite{HOTEL20182472}

    \item {\bf Principle Component Analysis PCA} Principal component analysis (PCA) is one of the most widely used multivariate analysis techniques. The aim of PCA is to reduce the data to a few characteristic dimensions for visualisation and analysis. This is achieved by calculating a new set of variables, or principal components, each of which is a linear combination of the original variables. The principal components (PCs) are chosen so that the important information in the data is retained in just a few of these new variables, effectively summarising the samples or observations. Better approximations are obtained by using more PCs, where each successive PC is uncorrelated with the previous PCs and expresses as much of the remaining variance as possible. This is achieved from the eigenvalue decomposition of the data covariance matrix with the coefficients for the k'th principal component determined by the eigenvector of the covariance matrix corresponding to the kth largest eigenvalue. Data reduction is achieved by only keeping the first few PCs, which contain most of the information in the data \cite{mckenzie2011analysis}. 

    \item {\bf CNN based Encoders}
    Convolutional autoencoder extends the basic structure of the simple autoencoder by replacing the fully connected layers to convolution layers. Encoder use convolutions and decoders use transposed convolutional layers. CNN based encoders can reduce the less important feature coefficients to zero and the magnitude of the coefficients can be used as feature scores.
    \end{itemize}

{
\small
\printbibliography %Prints bibliography
% References follow the acknowledgments in the camera-ready paper. Use unnumbered first-level heading for
% the references. Any choice of citation style is acceptable as long as you are
% consistent. It is permissible to reduce the font size to \verb+small+ (9 point)
% when listing the references.
% Note that the Reference section does not count towards the page limit.
% \medskip
}

\end{document}